\documentclass[10pt,twocolumn,letterpaper]{article}

\usepackage[]{cvpr}

\usepackage[utf8]{inputenc}
\usepackage{graphicx}
\usepackage{amsmath,amssymb,amsfonts, amsthm}
\usepackage{times}
\usepackage{epsfig}
\usepackage{array}
\usepackage{caption}
\usepackage{subcaption}
\usepackage{mdframed}
\usepackage{booktabs}
\usepackage{multirow}
\usepackage[pagebackref,breaklinks,colorlinks]{hyperref}
\usepackage{comment}
\usepackage{mathtools}
\usepackage[capitalize]{cleveref}
\crefname{section}{Sec.}{Secs.}
\Crefname{section}{Section}{Sections}
\Crefname{table}{Table}{Tables}
\crefname{table}{Tab.}{Tabs.}

\def\cvprPaperID{} 
\def\confName{}
\def\confYear{2022}

\newcommand{\norm}[1]{\left\lVert #1 \right\rVert}
\newtheorem{theorem}{Theorem}
\newtheorem{lemma}{Lemma}
\newtheorem{corollary}{Corollary}
\newtheorem{assumption}{Assumption}

\newenvironment{customthm}[1]
  {\renewcommand\thetheorem{#1}\theorem}
  {\endtheorem}
\newenvironment{customcoll}[1]
  {\renewcommand\thecorollary{#1}\corollary}
  {\endcorollary}

\makeatletter
\newcommand{\subalign}[1]{%
  \vcenter{%
    \Let@ \restore@math@cr \default@tag
    \baselineskip\fontdimen10 \scriptfont\tw@
    \advance\baselineskip\fontdimen12 \scriptfont\tw@
    \lineskip\thr@@\fontdimen8 \scriptfont\thr@@
    \lineskiplimit\lineskip
    \ialign{\hfil$\m@th\scriptstyle##$&$\m@th\scriptstyle{}##$\hfil\crcr
      #1\crcr
    }%
  }%
}
\makeatother

\usepackage{xcolor}

\DeclareMathOperator*{\argmin}{arg\,min}

\title{Non-Uniform Diffusion Models}

\author{Georgios Batzolis\\
DAMTP, \ University of Cambridge\\
{\tt\small gb511@cam.ac.uk}
\and
Jan Stanczuk\\
DAMTP, \ University of Cambridge\\
{\tt\small js2164@cam.ac.uk}
\and
Carola-Bibiane Schönlieb\\
DAMTP, \ University of Cambridge\\
{\tt\small cbs31@cam.ac.uk}
\and
Christian Etmann\\
Deep Render\\
{\tt\small christian.etmann@deeprender.ai}
}

\date{October 2021}

\makeatletter
\DeclareRobustCommand
  \myvdots{\vbox{\baselineskip4\p@ \lineskiplimit\z@
    \hbox{.}\hbox{.}\hbox{.}}}
\makeatother

\DeclareRobustCommand
  \Compactcdots{\mathinner{\cdotp\mkern-2mu\cdotp\mkern-2mu\cdotp}}
  
\begin{document}

\maketitle

\begin{abstract}
    Diffusion models have emerged as one of the most promising frameworks for deep generative modeling. In this work, we explore the potential of non-uniform diffusion models. We show that non-uniform diffusion leads to multi-scale diffusion models which have similar structure to this of multi-scale normalizing flows. We experimentally find that in the same or less training time, the multi-scale diffusion model achieves better FID score than the standard uniform diffusion model. More importantly, it generates samples $4.4$ times faster in $128\times 128$ resolution. The speed-up is expected to be higher in higher resolutions where more scales are used. Moreover, we show that non-uniform diffusion leads to a novel estimator for the conditional score function which achieves on par performance with the state-of-the-art conditional denoising estimator. Our theoretical and experimental findings are accompanied by an open source library \texttt{MSDiff} which can facilitate further research of non-uniform diffusion models.
\end{abstract}

\begin{figure}[h]
    \begin{center}
    \begin{tabular}{ccc}
        \scriptsize Original image $x$ & \scriptsize  Observation $y$ &  \scriptsize  Sample from $p_\theta(x|y)$  \\

        \includegraphics[width=.13\textwidth]{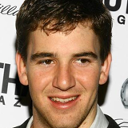} &   
        \includegraphics[width=.13\textwidth]{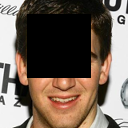} &
        \includegraphics[width=.13\textwidth]{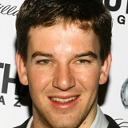}  \\

        \includegraphics[width=.13\textwidth]{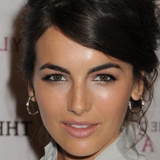} &   
        \includegraphics[width=.13\textwidth]{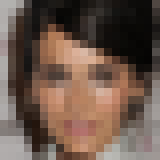} &
        \includegraphics[width=.13\textwidth]{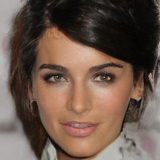}  \\

        \includegraphics[width=.13\textwidth]{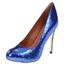} &   
        \includegraphics[width=.13\textwidth]{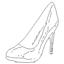} &
        \includegraphics[width=.13\textwidth]{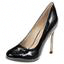}  \\
    \end{tabular}
    \end{center}
    \caption{Results from our conditional multi-speed diffusive estimator.}
    \label{fig: teaser}
\end{figure}

\section{Introduction}
The goal of generative modelling is to learn a  probability distribution from a finite set of samples. This classical problem in statistics has been studied for many decades, but until recently efficient learning of high-dimensional distributions remained impossible in practice. For images, the strong inductive biases of convolutional neural networks have recently enabled the modelling of such distributions, giving rise to the field of deep generative modelling.

Deep generative modelling became one of the central areas of deep learning with many successful applications.
In recent years much progress has been made in unconditional and conditional image generation.
The most prominent approaches are auto-regressive models \cite{bengio2005autoregressive}, variational auto-encoders (VAEs) \cite{kingma2014autoencoding},  normalizing flows \cite{papamakarios2021normalizing} and generative adversarial networks (GANs) \cite{goodfellow2014generative}.

Despite their success, each of the above methods suffers from important limitations. Auto-regressive models allow for likelihood estimation and high-fidelity image generation, but suffer from poor time complexity in high resolutions. VAEs and normalizing flows are less computationally expensive and allow for likelihood estimation, but tend to produce samples of lower visual quality. Moreover, normalizing flows put restrictions on the possible model architectures (requiring invertibility of the network and a Jacobian log-determinant that is computationally tractable), thus limiting their expressivity. While GANs produce state-of-the art quality samples, they don't allow for likelihood estimation and are notoriously hard to train due to training instabilities and mode collapse. 

Recently, score-based \cite{hyvarinen2005score_original} and diffusion-based  \cite{sohldickstein2015diffusion_original} generative models have been revived and improved in \cite{song2020generative_score} and \cite{ho2020denoising}.  
The connection between the two frameworks in discrete-time formulation has been discovered in \cite{vincent2011connection}. 
Recently in \cite{song2021sde}, both frameworks have been unified into a single continuous-time approach based on stochastic differential equations \cite{song2021sde} and are called score-based diffusion models. 
These approaches have recently received a lot of attention, achieving state-of-the-art performance in likelihood estimation \cite{song2021sde} and unconditional image generation \cite{dhariwal2021diffusion_beats_gans}, surpassing even the celebrated success of GANs.

In addition to achieving state-of-the art performance in both image generation and likelihood estimation, score-based diffusion models don't suffer from training instabilities or mode collapse \cite{dhariwal2021diffusion_beats_gans, song2021sde}. However, although their time complexity in high resolutions is better than that of auto-regressive models \cite{dhariwal2021diffusion_beats_gans}, it is still notably worse to that of GANs, normalizing flows and VAEs. Despite the recent efforts to close the sampling time gap between diffusion models and the faster frameworks, diffusion models still require significantly more time to achieve equal performance.

In this work, we explore non-uniform diffusion models. In non-uniform diffusion models, different parts of the input tensor diffuse with different diffusion speeds or more generally according to different stochastic differential equations. We find that the generalization of the original uniform diffusion framework can lead to multi-scale diffusion models which achieve improved sampling performance at a significantly faster sampling speed.

Moreover, we find that non-uniform diffusion can be used for conditional generation, because it leads to a novel estimator of the conditional score. We conduct a review and classification of existing approaches and perform a systematic comparison to find the best way of estimating the conditional score. We provide a proof of validity for the \textit{conditional denoising estimator} (which has been used in \cite{saharia2021sr3,tashiro2021csdi} without justification), and we thereby provide a firm theoretical foundation for using it in future research.

\noindent
\textbf{The contributions of this paper are as follows:}
\begin{enumerate}
    \item We introduce a principled objective for training non-uniform diffusion models.
    \item We show that non-uniform diffusion leads to the multi-scale diffusion models which are more efficient than uniform diffusion models. In less training time, the multi-scale models reach improved FID scores with significantly faster sampling speed. The speed up factor is expected to increase as we increase the number of scales.
    \item We show that non-uniform diffusion leads to \textit{conditional multi-speed diffusive estimator} (CMDE), a novel estimator of conditional score, which unifies previous methods of conditional score estimation.
    \item We provide a proof of consistency for the \textit{conditional denoising estimator} - one of the most successful approaches to estimating the conditional score. 
    
    \item We review and empirically compare score-based diffusion  approaches to modelling conditional distributions of image data. The models are evaluated on the tasks of super-resolution, inpainting and edge to image translation.

    \item We provide an open-source library \texttt{MSDiff}, to facilitate further research on conditional and non-uniform diffusion models. \footnote{The code will be released in the near future.}
\end{enumerate}

\section{Notation}

In this work we will use the following notation:
\begin{itemize}
    \item \textbf{Functions of time}
    \begin{gather*}
        f_t := f(t)
    \end{gather*}
    \item \textbf{Indexing vectors} \\
    Let $v = (v_1, ..., v_n) \in \mathbb{R}^n$ and let $ 1 \leq i < j < n$. Then:
    \begin{align*}
        v[:j] &:= (v_1, v_{2}, ..., v_j) \in \mathbb{R}^{j},
    \end{align*}
    cf. Section \ref{sec:CDiffE}.
    \item \textbf{Probability distributions} \\
    We denote the probability distribution of a random variable solely via the name of its density's argument, e.g.
    \begin{gather*}
        p(x_t) := p_{X_t}(x_t),
    \end{gather*}   
    where $x_t$ is a realisation of the random variable $X_t$.
    \item \textbf{Iterated Expectations}
    \begin{align*}
        &\mathbb{E}_{\subalign{z_1 &\sim p(z_1) \\ &{\myvdots} \\z_n &\sim p(z_n)}}[f(z_1,\Compactcdots,z_n)] \\
        :=&\mathbb{E}_{z_1 \sim p(z_1)} \Compactcdots \mathbb{E}_{z_n \sim p(z_n)}[f(z_1,\Compactcdots,z_n)]
    \end{align*}

\end{itemize} 

\section{Methods}
In the following, we will provide details about the framework and estimators discussed in this paper.
\subsection{Background: Score matching through Stochastic Differential Equations}
\subsubsection{Score-Based Diffusion}

In a recent work \cite{song2021sde} score-based  \cite{hyvarinen2005score_original, song2020generative_score} and diffusion-based \cite{sohldickstein2015diffusion_original, ho2020denoising} generative models have been unified into a single continuous-time score-based framework where the diffusion is driven by a stochastic differential equation.  This framework relies on Anderson's Theorem \cite{anderson1982reverse_time_sde}, which states that under certain Lipschitz conditions on $f : \mathbb{R}^{n_x} \times \mathbb{R} \xrightarrow{} \mathbb{R}^{n_x}$ and $G : \mathbb{R}^{n_x} \times \mathbb{R}\xrightarrow{} \mathbb{R}^{n_x} \times \mathbb{R}^{n_x}$ and an integrability condition on the target distribution $p(\textbf{x}_0)$ a forward diffusion process governed by the following SDE:
\begin{gather}
\label{eq:forward_sde}
 d\textbf{x}_t = f(\textbf{x}_t,t)dt+G(\textbf{x}_t,t)d\textbf{w}_t  
\end{gather} 
has a reverse diffusion process governed by the following SDE:
\begin{multline}\label{eq:reverse_sde}
d\textbf{x}_t=[f(\textbf{x}_t,t)-G(\textbf{x}_t,t)G(\textbf{x}_t,t)^T\nabla_{\textbf{x}_t}{\ln{p_{\textbf{X}_t}(\textbf{x}_t)}}]dt \\+ G(\textbf{x}_t,t)d\Bar{\textbf{w}_t},
\end{multline}

\noindent where $\Bar{\textbf{w}_t}$ is a standard Wiener process in reverse time. 

The forward diffusion process transforms the \textit{target distribution} $p(\textbf{x}_0)$ to a \textit{diffused distribution} $p(\textbf{x}_T)$ after diffusion time $T$. By appropriately selecting the drift and the diffusion coefficients of the forward SDE, we can make sure that after sufficiently long time $T$, the diffused distribution $p(\textbf{x}_T)$ approximates a simple distribution, such as $\mathcal{N}(\textbf{0},\textbf{I})$. We refer to this simple distribution as the \textit{prior distribution}, denoted by $\pi$. The reverse diffusion process transforms the diffused distribution $p(\textbf{x}_T)$ to the data distribution $p(\textbf{x}_0)$ and the prior distribution $\pi$ to a distribution $p^{SDE}$. $p^{SDE}$ is close to $p(\textbf{x}_0)$ if the diffused distribution $p(\textbf{x}_T)$ is close to the prior distribution $\pi$. We get samples from $p^{SDE}$ by sampling from $\pi$ and simulating the reverse sde from time $T$ to time $0$.

To get samples by simulating the reverse SDE, we need access to the time-dependent score function $\nabla_{\textbf{x}_t}{\ln{p(\textbf{x}_t)}}$ for all $\textbf{x}_t$ and $t$. In practice, we approximate the time-dependent score function with a neural network $s_{\theta}(\textbf{x}_t,t) \approx \nabla_{\textbf{x}_t}{\ln{p(\textbf{x}_t)}}$ and simulate the reverse SDE in equation \ref{eq:approximated_reverse_sde} to map the prior distribution $\pi$ to $p^{SDE}_{\theta}$.

\begin{multline}\label{eq:approximated_reverse_sde}
d\textbf{x}_t=[f(\textbf{x}_t,t)-G(\textbf{x}_t,t)G(\textbf{x}_t,t)^Ts_{\theta}(\textbf{x}_t,t)]dt \\+ G(\textbf{x}_t,t)d\Bar{\textbf{w}_t},
\end{multline}If the prior distribution is close to the diffused distribution and the approximated score function is close to the ground truth score function, the modeled distribution  $p^{SDE}_{\theta}$ is provably close to the target distribution $p(\textbf{x}_0)$. This statement is formalised in the language of distributional distances in the next subsection.

\subsubsection{Uniform Diffusion Models}
Previous works \cite{ho2020denoising, song2020generative_score, dhariwal2021diffusion_beats_gans} used the same forward SDE for the diffusion of all the pixels. For this reason, we classify them as uniform diffusion models. In uniform diffusion models, the sde in equation \ref{eq:uniform_sde} describes the forward diffusion for all pixels in an image:

\begin{equation}\label{eq:uniform_sde}
dx_t=f(x_t,t)dt + g(t)d\Bar{w_t},
\end{equation}

We used unbold notation for the random variables to show that this equation describes diffusion in one dimension. For uniform diffusion models, the neural network $s_\theta(\textbf{x}_t,t)$ can be trained to approximate the score function $\nabla_{\textbf{x}_t}{\ln{p(\textbf{x}_t)}}$ by minimizing the weighted score matching objective
\begin{gather}
    \mathcal{L}_{SM}(\theta, \lambda(\cdot)) := \frac{1}{2} \mathbb{E}_{\subalign{&t \sim U(0,T)\\ &\textbf{x}_t \sim p(\textbf{x}_t)}} [\lambda(t) \norm{\nabla_{\textbf{x}_t}{\ln{p(\textbf{x}_t)}} - s_\theta(\textbf{x}_t,t)}_2^2]
\end{gather}
where $\lambda: [0,T] \xrightarrow{} \mathbb{R}_+$ is a positive weighting function.

However, the above quantity cannot be optimized directly since we don't have access to the ground truth score $\nabla_{\textbf{x}_t}{\ln{p(\textbf{x}_t)}}$. Therefore in practice, a different objective has to be used \cite{hyvarinen2005score_original, song2020generative_score, song2021sde}. In \cite{song2021sde}, the weighted denoising score-matching objective is used, which is defined as 

\begin{gather}\label{DSM for uniform diffusion models}
\begin{aligned}
    &\mathcal{L}_{DSM}(\theta, \lambda(\cdot)) := \\ 
    &\frac{1}{2} \mathbb{E}_{\subalign{&t \sim U(0,T)\\ &\textbf{x}_0 \sim p(\textbf{x}_0) \\ &\textbf{x}_t \sim p(\textbf{x}_t | x_0)}} [\lambda(t) \norm{\nabla_{\textbf{x}_t}{\ln{p(\textbf{x}_t | \textbf{x}_0)}} - s_\theta(\textbf{x}_t,t)}_2^2]
\end{aligned}
\end{gather}

The difference between DSM and SM is the replacement of the ground truth score which we do not know by the score of the perturbation kernel which we know analytically for many choices of forward SDEs. The choice of the weighted DSM objective is justified because the weighted DSM objective is equal to the SM objective up to a constant that does not depend on the parameters of the model $\theta$. The reader can refer to \cite{vincent2011connection} for the proof. 

The choice of the weighting function is also important, because it determines the quality of score-matching in different diffusion scales. A principled choice for the weighting function is $\lambda(t) = g(t)^2$, where $g(\cdot)$ is the diffusion coefficient of the forward SDE. This weighting function is called the likelihood weighting function \cite{song2021maximum}, because it ensures that we minimize an upper bound on the Kullback–Leibler divergence from the target distribution to the model distribution by minimizing the weighted DSM objective with this weighting. The previous statement is implied by the combination of inequality \ref{Likelihood Weighting for Uniform Diffusion Models} which is proven in \cite{song2021maximum} and the relationship between the DSM and SM objectives.

\begin{equation}\label{Likelihood Weighting for Uniform Diffusion Models}
D_{KL}(p(\textbf{x}_0)\parallel p^{SDE}_{\theta})\leq L_{SM}(\theta, g(\cdot)^2)+D_{KL}(p(\textbf{x}_T)\parallel \pi)
\end{equation}

Other weighting functions have also yielded very good results with particular choices of forward sdes. However, we do not have theoretical guarantees that alternative weightings would yield good results with arbitrary choices of forward sdes.

\subsection{Non-Uniform Diffusion Models}\label{Non-Uniform Diffusion Models}
In this section, we describe non-uniform diffusion models. We call them non-uniform to indicate that the forward diffusion of each pixel is potentially governed by a different SDE. Considering a vectorised form $x=\text{vec}(X)=[x^1, x^2,...,x^{mnc}]$ of an image $X\in [0,1]^m\times[0,1]^n\times[0,1]^c$, we assume that the diffusion of the $i^{th}$ pixel is governed by the following SDE:

\begin{equation}\label{non-uniform pixel diffusion}
    dx^{i}_t = f_i(x^i_t,t)dt+g_i(t)dw^i_t
\end{equation}

Equation \ref{non-uniform pixel diffusion} is a special case of the general It\^{o} SDE described in equation \ref{eq:forward_sde}, but provides more flexibility compared to uniform diffusion where all pixels diffuse independently according to the same SDE. The diffusion of the entire image vector is summarised by the following SDE:

\begin{equation}\label{non-uniform image diffusion}
    d\textbf{x}_t = f(\textbf{x}_t,t)dt+G(t)d\textbf{w}_t,
\end{equation}

\noindent where \(f(\textbf{x}_t)=[f_1(x^1,t),...,f_{mnc}(x^{mnc}_t,t)]\) and \(G(t)=\text{diag}([g_1(t),...,g_{mnc}(t)])\).

In this more general setup, the DSM objective as described in equation \ref{DSM for uniform diffusion models} must also be generalised. The positive weighting function $\lambda(\cdot)$ is replaced by a positive definite matrix $\Lambda(\cdot)$ which gives the form of the DSM objective for non-uniform diffusion models:

\begin{gather}\label{DSM for non-uniform diffusion models}
\begin{aligned}
    &\mathcal{L}_{DSM}(\theta, \Lambda(\cdot)) := \\ 
    &\frac{1}{2} \mathbb{E}_{\subalign{&t \sim U(0,T)\\ &\textbf{x}_0 \sim p(\textbf{x}_0) \\ &\textbf{x}_t \sim p(\textbf{x}_t | x_0)}} [\textbf{v}_{\theta}(\textbf{x}_0,\textbf{x}_t,t)^T\Lambda(t)\textbf{v}_{\theta}(\textbf{x}_0,\textbf{x}_t,t)],
\end{aligned}
\end{gather}

\noindent where \(\textbf{v}_{\theta}(\textbf{x}_0,\textbf{x}_t,t)=\nabla_{\textbf{x}_t}{\ln{p(\textbf{x}_t | \textbf{x}_0)}} - s_\theta(\textbf{x}_t,t)\)

We prove that a principled choice for the positive weighting matrix is $\Lambda_{MLE}(t)=G(t)G(t)^T$. We call it the likelihood weighting matrix for non-uniform diffusion because it ensures minimization of an upper bound to the KL divergence from the target distribution to the model distribution. The previous statement is summarised in Theorem \ref{Likelihood Weighting for Non-uniform diffusion} which is proved in section \ref{appendix:weighting} of the Appendix.

\begin{theorem}\label{Likelihood Weighting for Non-uniform diffusion}
    Let $p(\textbf{x}_t)$ denote the distribution implied by the forward SDE at time $t$ and $p^{SDE}_\theta(\textbf{x}_t)$ denote the distribution implied by the parametrized reverse SDE at time $t$. Then under regularity assumptions of  \cite[Theorem 1]{song2021maximum} we have that 
    \label{thm:multi-dim}
    \begin{align*}
        KL(p(\textbf{x}_0) | p^{SDE}_\theta(\textbf{x}_0)) \leq &KL(p(\textbf{x}_T) | \pi(\textbf{x}_T)) 
        \\ &+ \frac{1}{2} \mathbb{E}_{\subalign{&t \sim U(0,T)\\ &\textbf{x}_t \sim p(\textbf{x}_t)}} 
        [
            \textbf{v}^T G(t)G(t)^T \textbf{v}
        ],
    \end{align*}
where $\textbf{v}=\nabla_{\textbf{x}_t} \ln{p(\textbf{x}_t)} - s_\theta(\textbf{x}_t,t)$. 
\end{theorem}

\subsection{Application of Non-Uniform Diffusion in multi-scale diffusion} \label{sec:multi-scale diffusion}

We design the forward process so that different groups of pixels diffuse with different speeds which creates a multi-scale diffusion structure. The intuition stems from multi-scale normalising flows. Multi-scale normalising flows invertibly transform the input tensor to latent encodings of different scales by splitting the input tensor into two parts after transformation in each scale. The multi-scale structure in normalising flows is shown to lead to faster training and sampling without compromise in generated image quality. 

We intend to transfer this idea to score-based modeling by diffusing some parts of the tensor faster. There are many ways to split the image into different parts which diffuse faster. It has been experimentally discovered that cascaded diffusion models \cite{saharia2021sr3} yield improved results  compared to standard diffusion models. This gave us the intuition to use a multi-level haar transform to transform every image to $n$ high frequency scales $d_1,...,d_n$ (detail coefficients) and one low frequency scale $a_n$ (approximation coefficient). The natural generation order of the haar coefficients (in line with cascaded diffusion) is $a_n$, $d_n$, $d_{n-1}$, ..., $d_1$. For this reason, we choose to diffuse lower frequency coefficients slower than high frequency coefficients. More specifically, we design the forward process so that all coefficients reach the same signal-to-noise ratio at the end of their diffusion time. We set the diffusion time for $a_n$ to $T_{a_n}=1$ and for $d_i$ to $T_{d_i}=\frac{i}{n+1}$ for each $i\in [1,...,n]$.

\subsubsection{Training}
We approximate the score of the distribution of $c_i(t) = [a_n(t), d_n(t),...,d_i(t)]$ in the time range $[(i-1)/(n+1), i/(n+1)]$ with a separate neural network $s_i(c_i(t), t)$. We also use a separate network $s_{n+1}(c_{n+1}(t), t)$ to approximate the score of the distribution of $c_{n+1}(t)=a_n(t)$ in the diffusion time range $[n/(n+1), 1]$. We use different networks in each scale to leverage the fact that we approximate the score of lower dimensional distributions. This enables faster score function evaluation and, therefore, faster training and sampling. We train each network separately using the likelihood weighting matrix for non-uniform diffusion (see section \ref{Non-Uniform Diffusion Models}).

\subsubsection{Sampling}
The sampling process is summarised in the following steps:

\begin{enumerate}
    \item  Sample $a_n(1)$ from the stationary distribution (e.g. standard normal distribution) and integrate the reverse sde for $a_n$ from time $t=1$ to time $t=n/(n+1)$. 
    \item Sample $d_n(n/(n+1))$ from the stationary distribution and solve the reverse sde for $[a_n, d_n]$ from time $t=n/(n+1)$ to time $t=(n-1)/(n+1)$.
    \item The process is continued as implied until we reach the final generation level, where we sample $d_1(1/(n+1))$ from the stationary distribution and solve the reverse sde for $[a_n,d_n,...,d_1]$ from time $t=1/(n+1)$ to time $t=\epsilon$ (e.g., $\epsilon=10^{-5}$). 
    \item We convert the generated haar coefficients $[a_n(\epsilon),d_n(\epsilon),...,d_1(\epsilon)]$ to the generated image using the multi-level inverse haar transform.
\end{enumerate}

Our experimental results presented in section \ref{Multiscale diffusion experiments} show that multiscale diffusion is more efficient and effective than uniform diffusion.


\subsection{Application of Non-Uniform Diffusion in Conditional generation} \label{sec:conditional generation}

The continuous score-matching framework can be extended to conditional generation, as shown in  \cite{song2021sde}. Suppose we are interested in $p(x|y)$, where $x$ is a \textit{target image} and $y$ is a \textit{condition image}. Again, we use the forward diffusion process (Equation \ref{eq:forward_sde}) to obtain a family of diffused distributions $p(x_t | y)$ and apply Anderson's Theorem to derive the \textit{conditional reverse-time SDE}
\begin{equation}
    \label{eq:conditional_reverse_sde}
    dx = [\mu(x,t) - \sigma(t)^2 \nabla_{x} \ln p_{X_t}(x | y)]dt + \sigma(t)d\tilde{w}.
\end{equation}
Now we need to learn the score $\nabla_{x_t} \ln p(x_t|y)$ in order to be able to sample from $p(x | y)$ using reverse-time diffusion.\\

In this work, we discuss the following approaches to estimating the conditional score $\nabla_{x_t} \ln p(x_t|y)$:
\begin{enumerate}
    \item Conditional denoising estimators
    \item Conditional diffusive estimators
    \item Multi-speed conditional diffusive estimators (our method)
\end{enumerate}
We discuss each of them in a separate section.\\

In \cite{song2021sde} an additional approach to conditional score estimation was suggested:
This method proposes learning $\nabla_{x_t} \ln p(x_t)$ with an unconditional score model, and  learning $p(y | x_t)$ with an auxiliary model. Then, one can use 
    \begin{equation*}
        \nabla_{x_t}\ln p(x_t |y) = \nabla_{x_t} \ln p(x_t) + \nabla_{x_t} \ln p(y | x_t)
    \end{equation*}
to obtain $\nabla_{x_t} \ln p(x_t | y)$. Unlike other approaches, this requires training a separate model for $p(y|x_t)$. Appropriate choices of such models for tasks discussed in this paper have not been explored yet. Therefore we exclude this approach from our study. 

\subsubsection{Conditional denoising estimator (CDE)}
The conditional denoising estimator (CDE) is a way of estimating $p(x_t|y)$ using the denoising score matching approach \cite{vincent2011connection, song2020generative_score}. In order to approximate $p(x_t|y)$, the conditional denoising estimator minimizes
\begin{gather}
\begin{aligned}
        \label{CDN}
        &\frac{1}{2} \mathbb{E}_{\subalign{&t \sim U(0,T)\\ &x_0, y \sim p(x_0, y) \\ &x_t \sim p(x_t | x_0)}} 
        [\lambda(t) \norm{\nabla_{x_t} \ln{p(x_t | x_0)} - s_\theta(x_t, y, t)}_2^2]
\end{aligned}
\end{gather}
This estimator has been shown to be successful in previous works \cite{saharia2021sr3,tashiro2021csdi}, also confirmed in our experimental findings (cf. Section \ref{sec:experiments}). 

Despite the practical success, this estimator has previously been used without a theoretical justification of why training the above objective yields the desired conditional distribution. Since $p(x_t|y)$ does not appear in the training objective, it is not obvious that the minimizer approximates the correct quantity. 

By extending the arguments of \cite{vincent2011connection}, we provide a formal proof that the minimizer of the above loss does indeed approximate the correct conditional score $p(x_t|y)$. This is expressed in the following theorem.

\begin{theorem}
    \label{thm:CDE_consistency}
    The minimizer (in $\theta$) of
    \begin{gather*}
    \begin{aligned}
            \frac{1}{2} \mathbb{E}_{\subalign{&t \sim U(0,T)\\ &x_0, y \sim p(x_0, y) \\ &x_t \sim p(x_t | x_0)}} 
            [\lambda(t) \norm{\nabla_{x_t} \ln{p(x_t | x_0)} - s_\theta(x_t, y, t)}_2^2]
    \end{aligned}
    \end{gather*}    
    is the same as the minimizer of 
    \begin{gather*}
        \frac{1}{2} \mathbb{E}_{\subalign{&t \sim U(0,T)\\ &x_t, y \sim p(x_t, y)}} 
        [\lambda(t) \norm{\nabla_{x_t} \ln{p(x_t | y)} - s_\theta(x_t, y,t)}_2^2]
    \end{gather*}
\end{theorem}
\noindent
The proof for this statement can be found in Appendix \ref{appendix:minimizers}. 
Using the above theorem, the consistency of the estimator can be established.
\begin{corollary}
    Let $\theta^\ast$ be a minimizer of a Monte Carlo approximation of (\ref{CDN}), then (under technical assumptions, cf. Appendix \ref{appendix:consistency}) the conditional denoising estimator $s_{\theta^\ast}(x,y,t)$ is a consistent estimator of the conditional score $\nabla_{x_t} \ln p(x_t | y)$, i.e.
    \begin{gather*}
        s_{\theta^\ast}(x,y,t) \overset{P}{\to} \nabla_{x_t} \ln p(x_t | y)   
    \end{gather*}
    as the number of Monte Carlo samples approaches infinity.
\end{corollary}
\noindent
This follows from the previous theorem and the uniform law of large numbers. Proof in the Appendix \ref{appendix:consistency}.

\subsubsection{Conditional diffusive estimator (CDiffE)}
\label{sec:CDiffE}
Conditional diffusive estimators (CDiffE) have first been suggested in \cite{song2021sde}. The core idea is that instead of learning $p(x_t | y)$ directly, we diffuse both $x$ and $y$ and approximate $p(x_t | y_t)$, using the denoising score matching. Just like learning diffused distribution $\nabla_{x_t} \ln p(x_t)$ improves upon direct estimation of $\nabla_{x} \ln p(x)$ \cite{song2020generative_score, song2021sde}, diffusing both the input $x$ and condition $y$, and then learning $\nabla_{x_t} \ln p(x_t | y_t)$ could make optimization easier and give better results than learning  $\nabla_{x_t} \ln p(x_t | y)$ directly.
    
In order to learn $p(x_t | y_t)$, observe that
\begin{gather*}
    \nabla_{x_t}\ln p(x_t | y_t) = \nabla_{x_t}\ln p(x_t, y_t) = \nabla_{z_t}\ln p(z_t)[:n_x],
\end{gather*}
where $z_t := (x_t, y_t)$ and $n_x$ is the dimensionality of $x$. Therefore we can learn the (unconditional) score of the joint distribution $p(x_t, y_t)$ using the denoising score matching objective just like as in the unconditional case, i.e
\begin{gather}
    \label{CDF}
\begin{aligned}
    &\frac{1}{2} \mathbb{E}_{\subalign{&t \sim U(0,T)\\ &z_0 \sim p_0(z_0) \\ &z_t \sim p(z_t | z_0)}} [\lambda(t) \norm{\nabla_{z_t} \ln{p(z_t |z_0)} - s_\theta(z_t,t)}_2^2].
\end{aligned}
\end{gather}
We can then extract our approximation for the conditional score $\nabla_{x_t} \ln p(x_t|y_t)$ by simply taking the first $n_x$ components of $s_\theta(x_t, y_t,t)$.

The aim is to approximate $\nabla_{x_t} \ln p(x_t | y)$ with $\nabla_{x_t} \ln p(x_t|\hat{y_t})$, where $\hat{y}_t$ is a sample from $p(y_t | y)$. Of course this approximation is imperfect and introduces an error, which we call the \textit{approximation error}. CDiffE aims to achieve smaller optimization error by diffusing the condition $y$ and making the optimization landscape easier, at a cost of making this approximation error.

Now in order to obtain samples from the conditional distribution, we sample a point $x_T \sim \pi$ and integrate
\begin{gather*}
    dx = [\mu(x,t) - \sigma(t)^2 \nabla_{x} \ln p_{X_t|Y_t}(x | \hat{y}_t)]dt + \sigma(t)d\tilde{w}
\end{gather*}
from $T$ to $0$, sampling $\hat{y}_t \sim p(y_t | y)$ at each time step.



\subsubsection{Conditional multi-speed diffusive estimator (CMDE)}
\label{sec:CMDE}
\begin{figure}
\begin{mdframed}
    \begin{center}
       \textbf{Sources of error for different estimators}
    \end{center}
    \textbf{CDE}\\
    Optimization error:
    \begin{gather*}
    s_\theta(x,y,t) \approx \nabla_{x_t}\ln p(x_t|y)   
    \end{gather*}
    \textbf{CDiffE and CMDE}\\
    Optimization error:
    \begin{gather*}
        s_\theta(x,y,t) \approx \nabla_{x_t}\ln p(x_t|y_t)   
        \end{gather*}
    Approximation error:
    \begin{gather*}
        \nabla_{x_t}\ln p(x_t|\hat{y}_t) \approx \nabla_{x_t}\ln p(x_t|y)   
        \end{gather*}
    CDiffE aims to achieve smaller optimization error at a cost of higher approximation error. By controlling the diffusion speed of $y$, CMDE tries to find an optimal balance between optimization error and approximation error.
\end{mdframed}
\caption{Sources of error for different estimators}
\label{fig:box}
\end{figure}
In this section we present a novel estimator for the conditional score $\nabla_{x_t} \ln p(x_t | y)$ which we call the \textit{conditional multi-speed diffusive estimator} (CMDE). 

Our approach is based on two insights. Firstly, there is no reason why $x_t$ and $y_t$ in conditional diffusive estimation need to diffuse at the same rate. Secondly, by decreasing the diffusion rate of $y_t$ while keeping the diffusion speed of $x_t$ the same, we can bring $p(x_t |y_t)$ closer to $p(x_t |y)$, at the possible cost of making the optimization more difficult. This way we can \emph{interpolate} between the conditional denoising estimator and the conditional diffusive estimator and find an optimal balance between optimization error and approximation error (cf. Figure \ref{fig:box}). This can lead to a better performance, as indicated by our experimental findings (cf. Section \ref{sec:experiments}).

In our conditional multi-speed diffusive estimator, $x_t$ and $y_t$ diffuse according to SDEs with the same drift but different diffusion rates,
\begin{gather*}
    dx = \mu(x,t)dt+\sigma^x(t)dw  \\
    dy = \mu(y,t)dt+\sigma^y(t)dw.
\end{gather*}

Then, just like in the case of conditional diffusive estimator, we try to approximate the joint score $\nabla_{x_t, y_t} \ln p(x_t, y_t)$ with a neural network. Since $x_t$ and $y_t$ now diffuse according to different SDEs, we need to take this into account and replace the weighting function $\lambda(t):\mathbb{R} \xrightarrow{} \mathbb{R}_+ $ with a positive definite weighting matrix $\Lambda(t): \mathbb{R} \xrightarrow{} \mathbb{R}^{(n_x + n_y) \times (n_x + n_y)}$. Hence, the new training objective becomes
\begin{gather}
    \label{eq:CMDE}
    \begin{aligned}
        \frac{1}{2} \mathbb{E}_{\subalign{&t \sim U(0,T)\\ &z_0 \sim p_0(z_0) \\ &z_t \sim p(z_t | z_0)}} 
        [
            v^T \Lambda(t) v
        ],
    \end{aligned}
\end{gather}
where $v=\nabla_{z_t} \ln{p(z_t |z_0)} - s_\theta(z_t,t)$, $z_t=(x_t,y_t)$.

In \cite{song2021maximum} authors derive a likelihood weighting function $\lambda^{\text{MLE}}(t)$, which ensures that the objective of the score-based model upper-bounds the negative log-likelihood of the data, thus enabling approximate maximum likelihood training of score-based diffusion models. We generalize this result to the multi-speed diffusion case by providing a likelihood weighting matrix $\Lambda^{\text{MLE}}(t)$ with the same properties.

\begin{theorem}
    \label{thm:weightning}
    Let $\mathcal{L}(\theta)$ be the CMDE training objective (Equation \ref{eq:CMDE}) with the following weighting:
    \begin{gather*}
        \Lambda^{\text{MLE}}_{i,j}(t) =  
        \begin{cases} 
            \sigma^x(t)^2, \text{ if } i=j, \ i \leq n_x \\ 
            \sigma^y(t)^2, \text{ if } i=j, \ n_x < i \leq n_y  \\
            0, \text{ otherwise}
        \end{cases}            
    \end{gather*}
    Then the joint negative log-likelihood is upper bounded (up to a constant in $\theta$) by the training objective of CMDE
    \begin{gather*}
        -\mathbb{E}_{(x,y) \sim p(x,y)}[\ln p_\theta(x,y)] \leq \mathcal{L}(\theta) + C.
    \end{gather*}
\noindent
\end{theorem}



\noindent
The proof can be found in Appendix \ref{appendix:weighting}. 

Moreover we show that the mean squared approximation error of a multi-speed diffusion model is upper bounded and the upper bound goes to zero as the diffusion speed of the condition $\sigma^y(t)$ approaches zero.
\begin{theorem}
    Fix $t$, $x_t$ and $y$. Under mild technical assumptions (cf. Appendix \ref{appendix:mse}) there exists a function  $E: \mathbb{R} \xrightarrow{} \mathbb{R}$ monotonically decreasing to $0$, such that
    \begin{gather*}
        \mathbb{E}_{y_t \sim p(y_t|y)}[
            \norm{ \nabla_{x_t} \ln p(x_t|y_t) - \nabla_{x_t} \ln p(x_t|y)}_2^2
            ] \\
            \leq E(1/\sigma^y(t)).
    \end{gather*}
\end{theorem}
\noindent
The proof can be found in Appendix \ref{appendix:mse}.

Thus we see that the objective of CMDE approaches that of CDE as  $\sigma^y(t) \to 0$, and  CMDE coincides with CDiffE when $ \sigma^y(t) = \sigma^x(t)$ (cf. Figure \ref{fig:box}).

We experimented with different configurations of $\sigma^x(t)$ and $\sigma^y(t)$ and found configurations that lead to  improvements upon CDiffE and  CDE in certain tasks. The experimental results are discussed in detail in Section \ref{sec:experiments}.

\section{Experiments}
\label{sec:experiments}

\subsection{Multiscale diffusion}\label{Multiscale diffusion experiments}
In this part of the experimental section, we compare the performance of the multiscale model that depends on non-uniform pixel diffusion to the performance of the standard model that depends on uniform diffusion. We train and evaluate both models on Imagenet $128\times 128$ and CelebA-HQ $128\times 128$.

For the standard diffusion model, we use the beta-linear VP SDE \cite{ho2020denoising} and train the score model using the simple objective \cite{dhariwal2021diffusion_beats_gans} because it is experimentally shown to favor generation quality. The architecture of the score model follows the architecture of \cite{dhariwal2021diffusion_beats_gans}.

For the multiscale model, we use 3-level haar transform to transform the original images, which means that we create a multiscale model with four scales. For this reason, we use four score models $s_{\theta_1},s_{\theta_2},s_{\theta_3},s_{\theta_4}$ which approximate the score function in following diffusion ranges respectively $[\epsilon, 0.25], [0.25, 0.50], [0.50, 0.75], [0.75, 1]$. The reason we do not use $s_{\theta_1}$ to approximate the score function for the entire diffusion is that we stop the diffusion of the highest frequency detail coefficients $d_1$ at time $0.25$, as they reach the target minimum SNR (by design of the forward SDE). The remaining diffusing tensor has a quarter of the dimensionality of the original tensor. Therefore, we need a less expressive neural network to approximate the score function in the next diffusion time range. The architecture of all models follows the architecture of \cite{dhariwal2021diffusion_beats_gans}. We choose the number of base channels and the depth of the multiscale score models so that the total number of parameters of the multiscale model is approximately equal to the number of parameters of the standard diffusion model to ensure fair comparison. For the diffusion of each haar coefficient, we use a variance preserving process with log-linear SNR profile. We choose the maximum SNR (at $t=\epsilon$) and the minimum SNR (achieved at the terminal diffusion time for each coefficient) to match the maximum SNR and minimum SNR of the standard model respectively.

We evaluate both models using the FID score on 50K samples. We generate each sample by numerically integrating the reverse SDE with 256 total euler-maruyama steps. Our results (see tables \ref{tbl:ImageNet}, \ref{tbl:CelebA}) show that for the same training time, the multiscale model achieves better FID score with significantly faster sampling speed ($4.4$ times faster). In fact, our results on ImageNet show that the multiscale model achieves improved FID score with faster sampling speed and less training time. The FID scores are higher than reported scores in prior work for both the multiscale and the standard model because we did not use Tweedie's formula for denoising the last step. We verified that by integrating the corresponding probability flow ODEs using the euler method. In that case, we got lower FID scores for both methods but the relative performance remained the same. Moreover, we used lighter neural networks than prior works to approximate the score function which led to generally worse performance. We opted for lighter models in this study because we wanted to conduct a fair comparison of the multiscale diffusion model and the standard uniform diffusion model. Improved techniques that led to state-of-the-art performance of the uniform diffusion model such as class conditioning and learning of the variance schedule \cite{dhariwal2021diffusion_beats_gans} can also be readily employed in the multiscale model. Given the superiority of the multiscale model, we expect the employment of improved techniques to further improve the performance of the multiscale model and potentially redefine the state-of-the-art. We intend to explore this direction in the future.

The training and sampling speed-up is attributed to the fact that we approximate the score of lower dimensional distributions for the majority of the diffusion. Therefore, we expect higher relative speed-ups in higher resolutions. We believe that the effectiveness of the multiscale model is attributed to the effectiveness of cascaded diffusion observed in previous works \cite{saharia2021sr3, dhariwal2021diffusion_beats_gans}. The difference between our multiscale model and the previous works is that it does not suffer from the effect of the compounding error. Ho et al. \cite{saharia2021sr3} improve the performance of cascading models by using an expensive post-training tuning step which they call conditioning augmentation. Our multiscale model essentially employs a cascading modeling structure that does not require any post-training tuning for improved sample generation.

\begin{table*}[h!]
    \begin{center}
    \caption{Multiscale and Vanilla model comparison on ImageNet 128x128}
    \label{tbl:ImageNet}
    \begin{tabular}{ccccccc}
    \toprule
    & Iterations & Parameters & Training (hours) $\downarrow$  & Sampling (secs) $\downarrow$ & FID $\downarrow$ \\
    \midrule
    \multirow{1}{*}{Vanilla} 
    &1M & 100M & 191.0  & 53.5 & 79.21    \\
    \midrule
    \multirow{1}{*}{Multiscale}
    &2M & 100M & 136.6  & 12.1 & 70.87 \\

    \midrule
    \multirow{1}{*}{Multiscale +}
    &2M & 200M & 151.0   & 18.7 & 65.50  \\
    \bottomrule
    \end{tabular}
    \end{center}
\end{table*}

\begin{table*}[h!]
    \begin{center}
    \caption{Multiscale and Vanilla model comparison on CelebA-HQ 128x128}
    \label{tbl:CelebA}
    \begin{tabular}{ccccccc}
    \toprule
    & Iterations & Parameters & Training (hours) $\downarrow$  & Sampling (secs) $\downarrow$ & FID $\downarrow$ \\
    \midrule
    \multirow{1}{*}{Vanilla} 
    &0.67M & 100M & 128  & 53.5 & 54.3    \\
    \midrule
    \multirow{1}{*}{Multiscale}
    &2.54M & 100M & 128  & 12.1 & 31.8 \\

    \midrule
    \multirow{1}{*}{Multiscale +}
    &1.76M & 200M & 128  & 18.7 & 33.5  \\
    \bottomrule
    \end{tabular}
    \end{center}
\end{table*}

\subsection{Conditional Generation}
\begin{figure*}
    \captionsetup[subfigure]{labelformat=empty}
    \begin{subfigure}{.135\textwidth}
        \includegraphics[width=\textwidth]{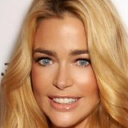}
        \caption{\scriptsize Original image $x$}
    \end{subfigure}
    \begin{subfigure}{.135\textwidth}
        \includegraphics[width=\textwidth]{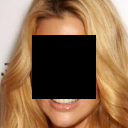}
        \caption{\scriptsize Observation $y := Ax$}
    \end{subfigure}
    \begin{subfigure}{.135\textwidth}
        \includegraphics[width=\textwidth]{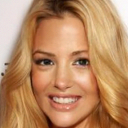}
        \caption{\scriptsize Reconstruction $\hat{x}_1$}
    \end{subfigure} 
    \begin{subfigure}{.135\textwidth}
        \includegraphics[width=\textwidth]{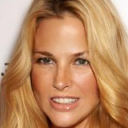}
        \caption{\scriptsize Reconstruction $\hat{x}_2$}
    \end{subfigure}    
    \begin{subfigure}{.135\textwidth}
        \includegraphics[width=\textwidth]{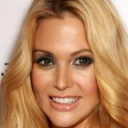}
        \caption{\scriptsize Reconstruction $\hat{x}_3$}
    \end{subfigure}    
    \begin{subfigure}{.135\textwidth}
        \includegraphics[width=\textwidth]{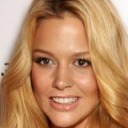}
        \caption{\scriptsize Reconstruction $\hat{x}_4$}
    \end{subfigure}
    \begin{subfigure}{.135\textwidth}
        \includegraphics[width=\textwidth]{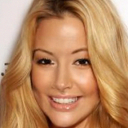}
        \caption{\scriptsize Reconstruction $\hat{x}_5$}
    \end{subfigure} 
    \caption{Diversity of five different CMDE reconstructions for a given masked image.}
\end{figure*}

In this section we conduct a systematic comparison of different score-based diffusion approaches to modelling conditional distributions of image data. We evaluate these approaches on the tasks of super-resolution, inpainting and edge to image translation.

\noindent
\textbf{Datasets} In our experiments, we use the CelebA \cite{2015celeba} and Edges2shoes \cite{yu2014sketch2shoe,isola2018pix2pix} datasets. We pre-processed the CelebA dataset as in \cite{liang2021hrflow}.

\noindent
\textbf{Models and hyperparameters} In order to ensure the fair comparison, we separate the evaluation of a particular estimator of conditional score from the evaluation of a particular neural network model. To this end, we train the same neural network architecture for all estimators. The architecture is based on the DDPM model used in \cite{ho2020denoising, song2021sde}. 
We used the variance-exploding SDE \cite{song2021sde} given by: 
$$dx = \sqrt{\frac{d}{dt}\sigma^2(t)}dw, \hspace{1cm}
\sigma(t) = \sigma_{min} \left(\frac{\sigma_{max}}{\sigma_{min}}\right)^t$$
Likelihood weighting was employed for all experiments. For CMDE, the diffusion speed of $y$ was controlled by picking an appropriate $\sigma^y_{\max}$, which we found by trial-and-error. The performance of CMDE could be potentially improved by performing a systematic hyperparameter search for optimal $\sigma^y_{\max}$.
Details on hyperparameters and architectures used in our experiments can be found in Appendix \ref{appendix:hyperparams}.

\noindent 
\textbf{Inverse problems} The tasks of inpainting, super-resolution and edge to image translation are special cases of inverse problems \cite{arridge2019ip, muller2012ip}. In each case, we are given a (possibly random) forward operator $A$ which maps our data $x$ (full image) to an observation $y$ (masked image, compressed image, sketch). The task is to come up with a high-quality reconstruction $\hat{x}$ of the image $x$ based on an observation $y$. The problem of reconstructing $x$ from $y$ is typically ill-posed, since $y$ does not contain all information about $x$. Therefore, an ideal algorithm would produce a reconstruction $\hat{x}$, which looks like a realistic image (i.e. is a likely sample from $p(x)$) and is consistent with the observation $y$ (i.e. $A\hat{x} \approx y$). Notice that if a conditional score model learns the conditional distribution correctly, then our reconstruction $\hat{x}$ is a sample from the posterior distribution $p(x | y)$, which satisfies bespoke requirements. This strategy for solving inverse problems is generally referred to as \emph{posterior sampling}.


\noindent
\textbf{Evaluation: Reconstruction quality} Ill-posedness often means that we should not strive to reconstruct $x$ perfectly. Nonetheless reconstruction error does correlate with the performance of the algorithm and has been one of the most widely-used metrics in the community. To evaluate the reconstruction quality for each task, we measure the Peak signal-to-noise ratio (PSNR) \cite{zhou2004psnr+ssim}, Structural similarity index measure (SSIM) \cite{zhou2004psnr+ssim} and Learned Perceptual Image Patch Similarity (LPIPS) \cite{zhang2018lpips} between the original image $x$ and the reconstruction $\hat{x}$.

\noindent
\textbf{Evaluation: Consistency}
In order to evaluate the consistency of the reconstruction, for each task we calculate the PSNR between $y:=Ax$ and $\hat{y}:=A\hat{x}$. 

\noindent
\textbf{Evaluation: Diversity}
We evaluate diversity of each approach by generating five reconstructions $(\hat{x})_{i=1}^5$ for a given observation ${y}$. Then for each $y$ we calculate the average standard deviation for each pixel among the reconstructions  $(\hat{x})_{i=1}^5$ . Finally, we average this quality over 5000 test observations.

\noindent
\textbf{Evaluation: Distributional distances}
If our algorithm generates realistic reconstructions while preserving diversity, then the distribution of reconstructions $p(\hat{x})$ should be similar to the distribution of original images $p(x)$. Therefore, we measure the Fr\'{e}chet Inception Distance (FID) \cite{heusel2018fid} between unconditional distributions $p(x)$ and $p(\hat{x})$ based on 5000 samples. Moreover, we calculate the FID score between the joint distributions $p(\hat{x}, y)$ and $p(x,y)$, which allows us to simultaneously check the realism of the reconstructions and the consistency with the observation. 
We use abbreviation UFID to refer to FID between between unconditional distributions and JFID to refer to FID between joints.  In our judgement, FID and especially the JFID is the most principled of the used metrics, since it measures how far $p_\theta(x | y)$ is from $p(x|y)$.

\begin{table*}
    \begin{center}
    \caption{Results of conditional generation tasks.}
    \label{tbl:results}
    \begin{tabular}{cccccccc}
    \toprule
    &Estimator & PSNR/SSIM $\uparrow$  & LPIPS $\downarrow$ & UFID/JFID $\downarrow$ & Consistency $\uparrow$ & Diversity $\uparrow$ \\
    \midrule
    \multirow{3}{*}{Inpainting} 
    &CDE & \textbf{25.12}/\textbf{0.870}  & \textbf{0.042} & 13.07/18.06 & \textbf{28.54} & 4.79  \\
    &CDiffE & 23.07/0.844   & 0.057 & 13.28/19.25 &  26.61 & \textbf{6.52}   \\
    &CMDE ($\sigma^y_{max} = 1$) & 24.92/0.864  & 0.044 & \textbf{12.07/17.07} & 28.32 & 4.98  \\
    \midrule
    \multirow{4}{*}{Super-resolution} 
    &CDE & 23.80/0.650  & 0.114 & 10.36/15.77 & 54.18 & \textbf{8.51}  \\
    &CDiffE & 23.83/0.656  & 0.139 & 14.29/20.20 & 51.90 & 7.41  \\
    &CMDE ($\sigma^y_{max} = 0.5$) & 23.91/0.654  & 0.109 & \textbf{10.28/15.68} & 53.03 & 8.33  \\
    \midrule
    \multirow{3}{*}{Edge to image} 
    &CDE & \textbf{18.35/0.699}  & \textbf{0.156} & \textbf{11.87/21.31} & \textbf{10.45} & 14.40 \\
    &CDiffE & 10.00/0.365   & 0.350 & 33.41/55.22  & 7.78 & \textbf{43.45} \\
    &CMDE ($\sigma^y_{max} = 1$) & 18.16/0.692  & 0.158 & 12.62/22.09 & 10.38 & 15.20  \\
    \bottomrule
    \end{tabular}
    \end{center}
\end{table*}

\subsubsection{Inpainting}

We perform the inpainting experiment using CelebA dataset. In inpainting, the forward operator $A$ is an application of a given binary mask to an image $x$.  In our case, we made the task more difficult by using randomly placed (square) masks. Then the conditional score model is used to obtain a reconstruction $\hat{x}$ from the masked image $y$. We select the position of the mask uniformly at random and cover $25\%$ of the image. The quantitative results are summarised in Table \ref{tbl:results} and samples are presented in Figure \ref{fig:inpainting}. We observe that CDE and CMDE significantly outperform CDiffE in all metrics, with CDE having a small advantage over CMDE in terms of reconstruction error and consistency. On the other hand, CMDE achieves the best FID scores.


\begin{figure}
\renewcommand{\arraystretch}{1.25}
    \begin{tabular}{lccc}
        \begin{tabular}{@{}l@{}}
            Original image $x$
            \\[25pt]
        \end{tabular}
         & \includegraphics[width=.08\textwidth]{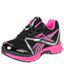} & \includegraphics[width=.08\textwidth]{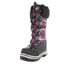} & \includegraphics[width=.08\textwidth]{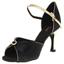} \\
        
        \begin{tabular}{@{}l@{}}
            Observation \\ $ y := Ax$
            \\[25pt]
        \end{tabular}
         & \includegraphics[width=.08\textwidth]{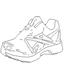} & \includegraphics[width=.08\textwidth]{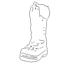} & \includegraphics[width=.08\textwidth]{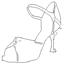} \\
        
         \begin{tabular}{@{}c@{}}
            CDE
            \\[25pt]
        \end{tabular} & \includegraphics[width=.08\textwidth]{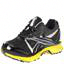}  & \includegraphics[width=.08\textwidth]{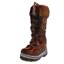}  & \includegraphics[width=.08\textwidth]{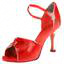} \\

        \begin{tabular}{@{}c@{}}
            CDiffE
            \\[25pt]
        \end{tabular} & \includegraphics[width=.08\textwidth]{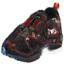} & \includegraphics[width=.08\textwidth]{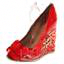} & \includegraphics[width=.08\textwidth]{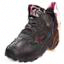}\\

        \begin{tabular}{@{}c@{}}
            CMDE (Ours)
            \\[25pt]
        \end{tabular}  &  \includegraphics[width=.08\textwidth]{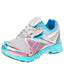} &  \includegraphics[width=.08\textwidth]{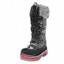} &  \includegraphics[width=.08\textwidth]{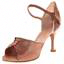} 
    \end{tabular}
    \caption{Edge to image translation results.}
    \label{fig:edges-to-shoes}
\end{figure}

\begin{figure*}
    \begin{center}
        \begingroup
        \setlength{\tabcolsep}{2pt}
    \begin{tabular}{ccccc}
        Original image $x$ & Observation $y$ & CDE & CDiffE & CMDE (Ours) \\

        \includegraphics[width=.15\textwidth]{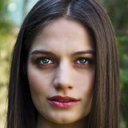} &   
        \includegraphics[width=.15\textwidth]{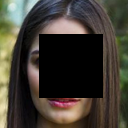} &
        \includegraphics[width=.15\textwidth]{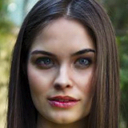} & 
        \includegraphics[width=.15\textwidth]{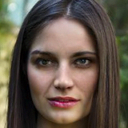} & 
        \includegraphics[width=.15\textwidth]{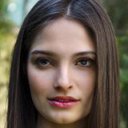} \\

        \includegraphics[width=.15\textwidth]{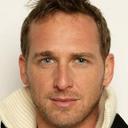} &   
        \includegraphics[width=.15\textwidth]{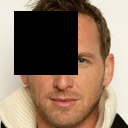} &
        \includegraphics[width=.15\textwidth]{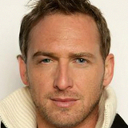} &
        \includegraphics[width=.15\textwidth]{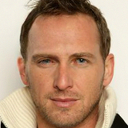} & 
        \includegraphics[width=.15\textwidth]{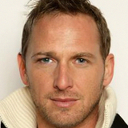} \\
 
        \includegraphics[width=.15\textwidth]{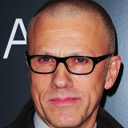} &   
        \includegraphics[width=.15\textwidth]{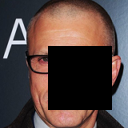} &
        \includegraphics[width=.15\textwidth]{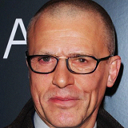} & 
        \includegraphics[width=.15\textwidth]{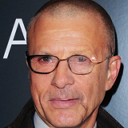} & 
        \includegraphics[width=.15\textwidth]{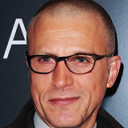} \\
    \end{tabular}
    \endgroup
    \end{center}
    \caption{Inpainting results.}
    \label{fig:inpainting}
\end{figure*}

\subsubsection{Super-resolution}
We perform 8x super-resolution using the CelebA dataset. A high resolution 160x160 pixel image $x$ is compressed to a low resolution 20x20 pixels image $y$. Here we use bicubic downscaling \cite{keyes1981bicubic} as the forward operator  $A$. Then using a score model we obtain a 160x160 pixel reconstruction image $\hat{x}$. The quantitative results are summarised in Table \ref{tbl:results} and samples are presented in Figure \ref{fig:super-resolution}. We find that CMDE and CDE perform similarly, while significantly outperforming CDiffE. CMDE achieves the smallest reconstruction error and captures the distribution most accurately according to FID scores.

\begin{figure*}
    \begin{center}
    \begingroup
    \setlength{\tabcolsep}{2pt}

    \begin{tabular}{cccccc}
        Original image $x$ & Observation $ y$ & HCFlow & CDE & CDiffE & CMDE (Ours) \\

        \includegraphics[width=.15\textwidth]{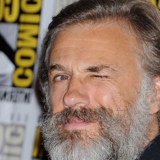} &   
        \includegraphics[width=.15\textwidth]{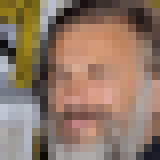} &
        \includegraphics[width=.15\textwidth]{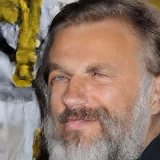} &
        \includegraphics[width=.15\textwidth]{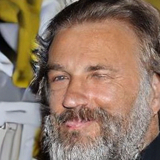} & 
        \includegraphics[width=.15\textwidth]{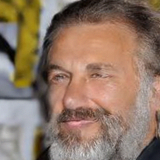} &
        \includegraphics[width=.15\textwidth]{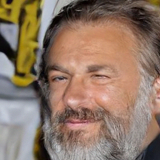} \\

        \includegraphics[width=.15\textwidth]{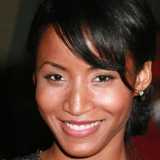} &   
        \includegraphics[width=.15\textwidth]{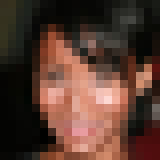} &
        \includegraphics[width=.15\textwidth]{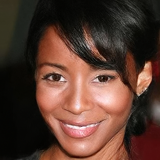} &
        \includegraphics[width=.15\textwidth]{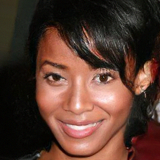} & 
        \includegraphics[width=.15\textwidth]{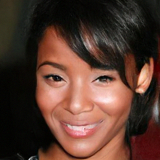} &
        \includegraphics[width=.15\textwidth]{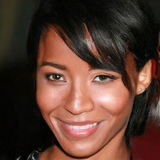} \\

        \includegraphics[width=.15\textwidth]{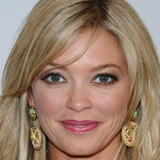} &   
        \includegraphics[width=.15\textwidth]{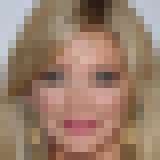} &
        \includegraphics[width=.15\textwidth]{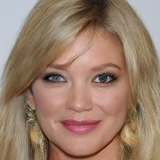} &
        \includegraphics[width=.15\textwidth]{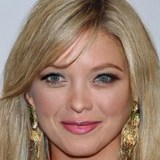} & 
        \includegraphics[width=.15\textwidth]{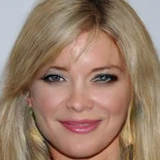} &
        \includegraphics[width=.15\textwidth]{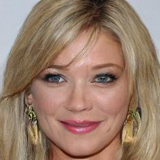} \\
    \end{tabular}
    \endgroup
    \end{center}
    \caption{Super-resolution results.}
    \label{fig:super-resolution}
\end{figure*}

\subsubsection{Edge to image translation}
We perform an edge to image translation task on the Edges2shoes dataset. The forward operator $A$ is given by a neural network edge detector \cite{xie2015edges}, which takes an original photo of a shoe $x$ and transforms it into a sketch $y$. Then a conditional score model is used to create an artificial photo of a shoe $\hat{x}$ matching the sketch. The quantitative results are summarised in Table \ref{tbl:results} and samples are presented in Figure \ref{fig:edges-to-shoes}. Unlike in inpainting and super-resolution where CDiffE achieved reasonable performance, in edge to image translation, it fails to create samples consistent with the condition (which leads to inflated diversity scores). CDE and CMDE are comparable, but CDE performed slightly better across all metrics. However, the performance of CMDE could be potentially improved by  tuning the diffusion speed $\sigma^y(t)$.

\section{Conclusions and future work}

In this article, we explored non-uniform diffusion models which rely on the idea of diffusing different parts of the tensor with different speeds or more generally according to different SDEs. We show that non-uniform diffusion leads to multiscale diffusion models which are more efficient and effective than standard uniform diffusion models for unconditional generation. More specifically, multiscale diffusion models achieve improved FID score with significantly faster sampling speed and for less training time.

We further discovered that non-uniform diffusion leads to CMDE, a novel estimator of the conditional score which can interpolate between conditional denoising estimator (CDE) and conditional diffusive estimator (CDiffE). We conducted a systematic comparison of different estimators of the conditional score and concluded that CMDE and CDE perform on par, while significantly outperforming CDiffE. This is particularly apparent in edge to image translation, where CDiffE fails to produce samples consistent with the condition image. Furthermore, CMDE outperformed CDE in terms of FID scores in inpainting and super-resolution tasks, which indicates that diffusing the condition at the appropriate speed can have beneficial effect on the optimization landscape, and yield better approximation of the posterior distribution. Furthermore, we provided theoretical analysis of the estimators of conditional score. More importantly, we proved the consistency of the conditional denoising estimator, thus providing a firm theoretical justification for using it in future research. 


\section{Acknowledgements} 
GB acknowledges the support from GSK and the Cantab Capital Institute for the Mathematics of Information. 
JS acknowledges the support from Aviva and the Cantab Capital Institute for the Mathematics of Information. 
CBS acknowledges support from the Philip Leverhulme Prize, the Royal Society Wolfson Fellowship, the EPSRC advanced career fellowship EP/V029428/1, EPSRC grants EP/S026045/1 and EP/T003553/1, EP/N014588/1, EP/T017961/1, the Wellcome Innovator Award RG98755, the Leverhulme Trust project Unveiling the invisible, the European Union Horizon 2020 research and innovation programme under the Marie Skodowska-Curie grant agreement No. 777826 NoMADS, the Cantab Capital Institute for the Mathematics of Information and the Alan Turing Institute. CE acknowledges support from
the Wellcome Innovator Award RG98755 for part of the work that was done at Cambridge.

\clearpage

{
\bibliography{bibliography}
\bibliographystyle{ieee_fullname}
}

\clearpage
\appendix

\section{Proofs}
\label{appendix:proofs}
\subsection{Equality of minimizers for CDE}
\label{appendix:minimizers}
\begin{lemma}
    \label{Vincent}
    For a fixed $y \in \mathbb{R}^d$ and $t \in \mathbb{R}$ we have
    \begin{align*}
        &\mathbb{E}_{\subalign{&x_0 \sim p(x_0 | y) \\ &x_t \sim p(x_t | x_0, y)}} 
            [\lambda(t) \norm{\nabla_{x_t} \ln{p(x_t | x_0, y)} - s_\theta(x_t, y, t)}_2^2]\\
        =& \mathbb{E}_{\subalign{&x_t \sim p(x_t |  y)}} 
            [\lambda(t) \norm{\nabla_{x_t} \ln{p(x_t | y)} - s_\theta(x_t, y, t)}_2^2]
    \end{align*}
\begin{proof}
    Since $y$ and $t$ are fixed, we may define $\psi(x_t) := s_\theta(x_t, y, t)$, $q(x_0) := p(x_0 | y)$ and $q(x_t | x_0) = p(x_t | x_0, y)$.
    Therefore, by the Tower Law, the statement of the lemma is equivalent to
    \begin{align*}
        &\mathbb{E}_{\subalign{&x_0, x_t \sim q(x_0, x_t)}} 
        [\norm{\nabla_{x_t} \ln{q(x_t | x_0)} - \psi(x_t)}_2^2]\\
    =& \mathbb{E}_{\subalign{&x_t \sim q(x_t)}} 
        [\norm{\nabla_{x_t} \ln{q(x_t)} - \psi(x_t)}_2^2]
    \end{align*}
    Which follows directly from \cite[Eq. 11]{vincent2011connection}.
\end{proof}
    
\end{lemma}
\begin{customthm}{1}
    The minimizer of
    \begin{gather*}
    \begin{aligned}
            &\frac{1}{2} \mathbb{E}_{\subalign{&t \sim U(0,T)\\ &x_0, y \sim p(x_0, y) \\ &x_t \sim p(x_t | x_0)}} 
            [\lambda(t) \norm{\nabla_{x_t} \ln{p(x_t | x_0)} - s_\theta(x_t, y, t)}_2^2]
    \end{aligned}
    \end{gather*}    
    in $\theta$ is the same as the minimizer of 
    \begin{gather*}
         \frac{1}{2} \mathbb{E}_{\subalign{&t \sim U(0,T)\\ &x_t, y \sim p(x_t, y)}} 
        [\lambda(t) \norm{\nabla_{x_t} \ln{p(x_t | y)} - s_\theta(x_t, y,t)}_2^2].
    \end{gather*}
\end{customthm}
\begin{proof}
    First, notice that $x_t$ is conditionally independent of $y$ given $x_0$. Therefore, by applying the Tower Law we obtain
    \begin{align*}
        &\mathbb{E}_{\subalign{&t \sim U(0,T)\\ &x_0, y \sim p(x_0, y) \\ &x_t \sim p(x_t | x_0)}} 
            [\lambda(t) \norm{\nabla_{x_t} \ln{p(x_t | x_0)} - s_\theta(x_t, y, t)}_2^2] \\
            \overset{(1)}{=}& \mathbb{E}_{\subalign{&t \sim U(0,T)\\ &y \sim p(y)\\&x_0 \sim p(x_0 | y) \\ &x_t \sim p(x_t | x_0)}} 
            [\lambda(t) \norm{\nabla_{x_t} \ln{p(x_t | x_0)} - s_\theta(x_t, y, t)}_2^2] \\
            \overset{(2)}{=}& \mathbb{E}_{\subalign{&t \sim U(0,T)\\ &y \sim p(y)\\&x_0 \sim p(x_0 | y) \\ &x_t \sim p(x_t | x_0, y)}} 
            [\lambda(t) \norm{\nabla_{x_t} \ln{p(x_t | x_0, y)} - s_\theta(x_t, y, t)}_2^2] \\ 
            =& \mathbb{E}_{\subalign{&t \sim U(0,T)\\ &y \sim p(y)}} 
            [f(t,y)] =: (*)
    \end{align*}
    where 
    \begin{align*}
        f&(t,y) := \\
        &\mathbb{E}_{\subalign{&x_0 \sim p(x_0 | y) \\ &x_t \sim p(x_t | x_0, y)}} [
        \lambda(t) \norm{\nabla_{x_t} \ln{p(x_t | x_0, y)} - s_\theta(x_t, y, t)}_2^2
    ].
    \end{align*}
    Now fix $y$ and $t$. By Lemma \ref{Vincent}, it follows that
    \begin{align*}
        &f(t,y) \\
        =&\mathbb{E}_{\subalign{&x_0 \sim p(x_0 | y) \\ &x_t \sim p(x_t | x_0, y)}} 
            [\lambda(t) \norm{\nabla_{x_t} \ln{p(x_t | x_0, y)} - s_\theta(x_t, y, t)}_2^2]\\
        \overset{(3)}{=}& \mathbb{E}_{\subalign{&x_t \sim p(x_t |  y)}} 
            [\lambda(t) \norm{\nabla_{x_t} \ln{p(x_t | y)} - s_\theta(x_t, y, t)}_2^2]\\
    \end{align*}
    Since $t$ and $y$ were arbitrary, this is true for all $t$ and $y$. Therefore, substituting back into $(*)$ we get that
    \begin{align*}
        (*) &= \mathbb{E}_{\subalign{&t \sim U(0,T)\\ &y \sim p(y)\\&x_t \sim p(x_t | y)}} 
            [\lambda(t) \norm{\nabla_{x_t} \ln{p(x_t | y)} - s_\theta(x_t, y, t)}_2^2]\\
        &\overset{(1)}{=} \mathbb{E}_{\subalign{&t \sim U(0,T)\\ &x_t, y \sim p(x_t, y)}} 
        [\lambda(t) \norm{\nabla_{x_t} \ln{p(x_t | y)} - s_\theta(x_t, y,t)}_2^2].
    \end{align*}
(1) Tower Law, (2) Conditional independence of $x_t$ and $y$ given $x_0$, (3) \text{Lemma} \ref{Vincent}.
\end{proof}

\subsection{Consistency of CDE}
\label{appendix:consistency}
In order to prove the consistency, in this subsection we make the following assumptions:
\begin{assumption}
    \label{assum:compact_1}
    The space of parameters $\Theta$ and the data space $\mathcal{X}$ are compact.
\end{assumption}
\begin{assumption}
    \label{assum:unique}
    There exists a unique $\theta^\ast \in \Theta$ such that $s_{\theta^\ast}(x,y,t) = \nabla_{x_t}\ln p (x,y,t)$.
\end{assumption}

First we state some technical, but well-known lemmas, which will be useful in proving our consistency result.

\begin{lemma}[Uniform law of large numbers] \cite[Lemma 2.4]{whitney1994estimation} \\
    \label{lemma:ULLN}
    Let $z_i$ be i.i.d from a distribution $q(z)$ and suppose that:
    \begin{itemize}
        \item $\Theta$ is compact.
        \item $f(z,\theta)$ is continuous for all $\theta \in \Theta$ and almost all $z$.
        \item $f(\cdot, \theta)$ is a measurable function of $z$ for each $\theta$.
        \item There exists $d: \mathcal{Z} \xrightarrow{} \mathbb{R}$ such that $\mathbb{E}[d(z)]<\infty$ and $\norm{f(z,\theta)} \leq d(z)$ for each $\theta$.
    \end{itemize}
    Then $\mathbb{E}_{z}[f(z,\theta)]$ is continuous in $\theta$, and $\frac{1}{n}\sum_{i=1}^n f(z_i, \theta)$ converges to $\mathbb{E}_{z}[f(z,\theta)]$ uniformly in probability, i.e.:
    \begin{gather*}
        \sup_\theta \norm{\frac{1}{n}\sum_{i=1}^n f(z_i, \theta) - \mathbb{E}_{z}[f(z,\theta)]} \overset{P}{\to} 0
    \end{gather*}

\end{lemma}

\begin{lemma}[Consistency of extremum estimators] \cite[Theorem 2.1]{whitney1994estimation} \\
    \label{lemma:consistency}
    Let $\Theta$ be compact and consider a family of functions $\mathcal{L}^{(n)}: \Theta \to \mathbb{R}$. Moreover, suppose there exists a function $\mathcal{L}: \Theta \to \mathbb{R}$ such that
    \begin{itemize}
        \item $\mathcal{L}(\theta)$ is uniquely minimized at $\theta^\ast$.
        \item $\mathcal{L}(\theta)$ is continuous.
        \item $\mathcal{L}^{(n)}(\theta)$ converges uniformly in probability to $\mathcal{L}(\theta)$.
    \end{itemize}
    Then $$\theta^\ast_n := \argmin_{\theta \in \Theta} \mathcal{L}^{(n)}(\theta) \overset{P}{\to} \theta^\ast.$$
\end{lemma}

\begin{customcoll}{1}
    Let $\theta_n^\ast$ be a minimizer of a $n$-sample Monte Carlo approximation of \begin{gather*}
        \begin{aligned}
                \frac{1}{2} \mathbb{E}_{\subalign{&t \sim U(0,T)\\ &x_0, y \sim p(x_0, y) \\ &x_t \sim p(x_t | x_0)}} 
                [\lambda(t) \norm{\nabla_{x_t} \ln{p(x_t | x_0)} - s_\theta(x_t, y, t)}_2^2].
        \end{aligned}
        \end{gather*} 
        Then under assumptions \ref{assum:compact_1} and \ref{assum:unique}, the conditional denoising estimator $s_{\theta_n^\ast}(x,y,t)$ is a consistent estimator of the conditional score $\nabla_{x_t} \ln p(x_t | y)$, i.e.
    \begin{gather*}
        s_{\theta_n^\ast}(x,y,t) \overset{P}{\to} \nabla_{x_t} \ln p(x_t | y),
    \end{gather*}
    as the number of Monte Carlo samples $n$ approaches infinity.
\end{customcoll}

\begin{proof}
    By conditional independence and the Tower Law, we get
    \begin{align*}
         &\mathbb{E}_{\subalign{&t \sim U(0,T)\\ &x_0, y \sim p(x_0, y) \\ &x_t \sim p(x_t | x_0)}} 
                [\lambda(t) \norm{\nabla_{x_t} \ln{p(x_t | x_0)} - s_\theta(x_t, y, t)}_2^2] \\
        =& \mathbb{E}_{\subalign{&t \sim U(0,T)\\ &x_0, y \sim p(x_0, y) \\ &x_t \sim p(x_t | x_0, y)}} 
            [\lambda(t) \norm{\nabla_{x_t} \ln{p(x_t | x_0)} - s_\theta(x_t, y, t)}_2^2] \\
        =& \mathbb{E}_{\subalign{&t \sim U(0,T)\\ &x_0,, x_t, y \sim p(x_0, x_t ,y)}} 
            [\lambda(t) \norm{\nabla_{x_t} \ln{p(x_t | x_0)} - s_\theta(x_t, y, t)}_2^2].
    \end{align*}
    Let $z=(t,x_0,x_t,y)$ and denote by $q(z):=p(t,x_0,x_t,y)$ the joint distribution. Moreover, define $f(z,\theta) := \lambda(t) \norm{\nabla_{x_t} \ln{p(x_t | x_0)} - s_\theta(x_t, y, t)}_2^2$. Since $t \sim U(0,T)$ is independent of $(x_0, x_t ,y) \sim p(x_0, x_t ,y)$, the above is equal to 
    \begin{align*}
        \mathbb{E}_{z \sim q(z)} 
            [f(z,\theta)]
    \end{align*}
    Therefore by Lemma \ref{lemma:ULLN}, the Monte Carlo approximation of \ref{CDN}: $\mathcal{L}^{(n)}(\theta)=\frac{1}{n}\sum_{i=1}^n f(z_i, \theta)$ converges uniformly in probability to $\mathcal{L}(\theta) = \mathbb{E}_{z \sim q(z)} 
    [f(z,\theta)]$. Let $\theta^\ast$ be the minimizer of $\mathcal{L}(\theta)$, by Lemma \ref{lemma:consistency} we get that $\theta^\ast_n \overset{P}{\to} \theta^\ast$ . Finally by Theorem \ref{thm:CDE_consistency}, $\theta^\ast$ is also a minimizer of the Fisher divergence between $s_{\theta^\ast}(x_t,y,t)$ and $\nabla_{x_t} \ln p(x_t | y)$ and by Assumption \ref{assum:unique} this implies that $s_{\theta^\ast}(x_t,y,t) = \nabla_{x_t} \ln p(x_t | y)$. Hence $s_{\theta_n^\ast}(x,y,t) \overset{P}{\to} \nabla_{x_t} \ln p(x_t | y)$.
\end{proof}

\subsection{Likelihood weighting for multi-speed and multi-sde models}
\label{appendix:weighting}
In this section we derive the likelihood weighting for multi-sde models (Theorem \ref{thm:weightning}). First using the framework in \cite[Appendix A]{song2021sde} we present the Anderson's theorem for multi-dimensional SDEs with non-homogeneous covariance matrix (without assuming $\Sigma(t) \not = \sigma(t) I$) and generalize the main result of \cite{song2021maximum} to this setting. Then, we cast the problem of multi-speed and  multi-sde diffusion as a special case of multi-dimensional diffusion with a particular covariance matrix $\Sigma(t)$ and thus obtain the likelihood weighting for multi-sde models (Theorem \ref{thm:weightning}).

Consider an Ito's SDE
\begin{align*}
    dx = \mu(x, t)dt + \Sigma(t)dw
\end{align*}
where $\mu: \mathbb{R}^{n_x} \times [0,T] \xrightarrow{} \mathbb{R}^{n_x}$ and $\Sigma: [0,T] \xrightarrow{} \mathbb{R}^{n_x \times n_x}$ is a time-dependent positive-definite matrix. By multi-dimensional Anderson's Theorem \cite{anderson1982reverse_time_sde} the corresponding reverse time SDE is given by 
\begin{align}
    \label{eq:true_rtsde}
    dx &= \tilde{\mu}(x, t)dt + \Sigma(t)dw \\
    \text{where } \tilde{\mu}(x,t) &:= \mu(x,t) -  \Sigma(t)^2 \nabla_x \ln p_{X_t}(x). \nonumber
\end{align}

If we train a score-based diffusion model to approximate $\nabla_x \ln p_{X_t}(x)$ with a neural network $s_\theta(x,t)$ we will obtain the following approximate reverse-time sde
\begin{align}
    \label{eq:approx_rtsde}
    dx &= \tilde{\mu}_\theta(x, t)dt + \Sigma(t)dw \\
    \text{where } \tilde{\mu}_\theta(x,t) &:= \mu(x,t) -  \Sigma(t)^2 s_\theta(x,t) \nonumber
\end{align}

Now we generalize \cite[Theorem 1]{song2021maximum} to multi-dimensional setting.
\begin{theorem}
    Let $p(x_t)$ and $p_\theta(x_t)$ denote marginal distributions of \ref{eq:true_rtsde} and \ref{eq:approx_rtsde} respectively. Then under regularity assumptions of  \cite[Theorem 1]{song2021maximum} we have that 
    \label{thm:multi-dim}
    \begin{align*}
        KL(p(x_0) | p_\theta(x_0)) \leq &KL(p(x_T) | \pi(x_T)) 
        \\ &+ \frac{1}{2} \mathbb{E}_{\subalign{&t \sim U(0,T)\\ &x_t \sim p(x_t)}} 
        [
            v^T \Sigma(t)^2 v
        ],
    \end{align*}
where $v=\nabla_{x_t} \ln{p(x_t)} - s_\theta(x_t,t)$. 
\end{theorem}
\begin{proof}   
    We proceed in close analogy to the proof of \cite[Theorem 1]{song2021maximum} but we use a more general diffusion matrix $\Sigma(t)$. 
    Let $P$ be the law of the true reverse-time sde and let $P_\theta$ be the law of the approximate reverse-time sde. 
    Then by  \cite[Theorem 2.4]{leonard2013properties} (generalized chain rule for KL divergence) we have
    \begin{align*}
        KL(P | P_\theta) = &KL(p(x_0) | p_\theta(x_0)) 
        \\ &+  \mathbb{E}_{z \sim p(x_0)}[KL(P(\cdot | x_0=z) |P_\theta(\cdot | x_0=z))].
    \end{align*}
    Since $\mathbb{E}_{z \sim p(x_0)}[KL(P(\cdot | x_0=z) |P_\theta(\cdot | x_0=z))] \geq 0$, this implies 
    \begin{align*}
       KL(p(x_0) | p_\theta(x_0)) \leq KL(P | P_\theta) 
    \end{align*}
    Using the fact that $p_\theta(x_T) = \pi$ and applying \cite[Theorem 2.4]{leonard2013properties} again, we obtain
    \begin{align*}
        KL(P | P_\theta) =  &KL(p(x_T) |\pi) 
        \\ &+  \mathbb{E}_{\subalign{z \sim p(x_T)}}[KL(P(\cdot | x_T=z) |P_\theta(\cdot | x_T=z))].
    \end{align*}
    Let $P^z := P(\cdot | x_T=z)$ and $P_\theta^z := P_\theta(\cdot | x_T=z)$
    \begin{align*}
        \mathbb{E}_{\subalign{z \sim p(x_T)}}[KL(P(\cdot | x_T=z) |P_\theta(\cdot | x_T=z))]
        \\ =  - \mathbb{E}_{\subalign{z \sim p(x_T)}} \left[ 
            \mathbb{E}_{P^z} \left[
                \ln \frac{d P^z_\theta}{d P^z}
            \right]
        \right]
    \end{align*}
    Using Girsanov Theorem \cite[Theorem 8.6.5]{oksendal2003sde} and the fact that $\Sigma(t)$ is symmetric and invertible 
    \begin{align*}
     =  &\mathbb{E}_{z \sim p(x_T)}  \bigg[ 
            \mathbb{E}_{P^z} \bigg[\\
                    &\int_0^T  \Sigma(t) v(x_t,t) dw_t 
                    + \frac{1}{2} \int_0^T v(x_t,t)^T \Sigma(t)^2 v(x_t,t) dt 
            \bigg]
        \bigg]
    \end{align*}
    where $v(x_t,t)=\nabla_{x_t} \ln{p(x_t)} - s_\theta(x_t,t)$. Since $\int_0^T  \Sigma(t) v(x_t,t) dw_t $ is a martingale (Ito's integral wrt Brownian motion) 
    \begin{align*}
        &=  \frac{1}{2} \mathbb{E}_{z \sim p(x_T)}  \bigg[ 
            \mathbb{E}_{P^z} \bigg[
                     \int_0^T v(x_t,t)^T \Sigma(t)^2 v(x_t,t) dt 
            \bigg]
        \bigg] \\
        &= \frac{1}{2} \int_0^T \mathbb{E}_{x \sim p(x_t)}[ v(x_t,t)^T \Sigma(t)^2 v(x_t,t)] \\
        &= \frac{1}{2} \mathbb{E}_{\subalign{&t \sim U(0,T)\\ &x_t \sim p(x_t)}} 
        [
            v(x_t,t)^T \Sigma(t)^2 v(x_t,t)
        ].
    \end{align*}
\end{proof}

\subsubsection{Multi-sde and multi-speed diffusion}
Now we consider again the multi-speed and the more general multi-sde diffusion frameworks from sections \ref{sec:multi-scale diffusion} and \ref{sec:conditional generation}. Suppose that we have two tensors $x$ and $y$ which diffuse according to different SDEs
\begin{gather*}
    dx = \mu^x(x,t)dt+\sigma^x(t)dw  \\
    dy = \mu^y(y,t)dt+\sigma^y(t)dw  
\end{gather*}
We may cast this system of two SDEs, as a single SDE
\begin{gather*}
    dz = \mu^z(z,t)dt+ \Sigma_z(t)dw 
\end{gather*}
where $z = (x,y)$, $\mu^z(z,t) = (\mu^x(x,t), \mu^y(x,t))$ and 
\begin{gather*}
    \Sigma_z(t) =  
    \begin{cases} 
        \sigma^x(t), \text{ if } i=j, \ i \leq n_x \\ 
        \sigma^y(t), \text{ if } i=j, \ n_x < i \leq n_y  \\
        0, \text{ otherwise}
    \end{cases}
    .            
\end{gather*}
If we train a score-based diffusion model for $z_t = (x_t, y_t)$, then by Theorem \ref{thm:multi-dim}
\begin{align*}
    KL(p(z_0) | p_\theta(z_0)) \leq C_1 + \frac{1}{2} \mathbb{E}_{\subalign{&t \sim U(0,T)\\ &z_t \sim p(z_t)}} 
    [
        v^T \Sigma_z(t)^2 v
    ],
\end{align*}
where $C_1 := KL(p(x_T) | \pi(x_T)) $ does not depend on $\theta$.
Because $\Lambda_{MLE}$ (from Theorem \ref{thm:weightning}) is equal to $\Sigma_z(t)^2$, we may rewrite the above as 
\begin{align*}
    KL(p(z_0) | p_\theta(z_0)) \leq C_1 + \frac{1}{2} \mathbb{E}_{\subalign{&t \sim U(0,T)\\ &z_t \sim p(z_t)}} 
    [
        v^T \Lambda_{MLE}(t)^2 v
    ],
\end{align*}
and since by denoising score matching \cite{vincent2011connection} 
\begin{align*}
    \mathbb{E}_{\subalign{&t \sim U(0,T)\\ &z_t \sim p(z_t)}} 
[
    v^T \Lambda_{MLE}(t) v
] &= 
\\ &\mathbb{E}_{\subalign{&t \sim U(0,T)\\ &z_0 \sim p_0(z_0) \\ &z_t \sim p(z_t | z_0)}} 
[
    v^T \Lambda_{MLE}(t) v
] + C_2
\end{align*}
where $C_2$ is another term constant in $\theta$.
We conclude that 
\begin{align*}
    KL(p(z_0) | p_\theta(z_0)) \leq \frac{1}{2} \mathbb{E}_{\subalign{&t \sim U(0,T)\\ &z_0 \sim p_0(z_0) \\ &z_t \sim p(z_t | z_0)}} 
    [
        v^T \Lambda_{MLE}(t) v
    ] + C_3
\end{align*}
where $C_3 := C_1 + C_2$.
Now recall that the term on the RHS is exactly the training objective of a multi-sde score-based diffusion model with likelihood weighting
\begin{gather*}
    \begin{aligned}
        \mathcal{L}(\theta) := \frac{1}{2} \mathbb{E}_{\subalign{&t \sim U(0,T)\\ &z_0 \sim p_0(z_0) \\ &z_t \sim p(z_t | z_0)}} 
        [
            v^T \Lambda_{MLE}(t) v
        ].
    \end{aligned}
\end{gather*}
Therefore
\begin{align*}
    KL(p(z_0) | p_\theta(z_0)) \leq \mathcal{L}(\theta)  + C_3.
\end{align*}
Finally, since $KL(p(z_0) | p_\theta(z_0))  = \mathbb{E}_{(x,y) \sim p(x,y)}[\ln p(x,y)]  - \mathbb{E}_{(x,y) \sim p(x,y)}[\ln p_\theta(x,y)]$, we have 
\begin{gather*}
    -\mathbb{E}_{(x,y) \sim p(x,y)}[\ln p_\theta(x,y)] \leq \mathcal{L}(\theta) + C
\end{gather*}
where $C := C_3 - \mathbb{E}_{(x,y) \sim p(x,y)}[\ln p(x,y)] $ is independent of $\theta$. Thus the Theorem \ref{thm:weightning} is established.

\subsection{Mean square approximation error}
\label{appendix:mse}

\begin{assumption}
    \label{assum: c2}
    $p(x,y) \in C^2(\mathcal{X})$\
\end{assumption}
\begin{assumption}
    \label{assum: lower_bound}
    $p(x,y) > 0$ for all $x,y$.
\end{assumption}
\begin{assumption}
    \label{assum: compact_2}
    The data space $\mathcal{X}$ is compact.
\end{assumption}

\begin{lemma}
    \label{lemma: blurring}
    Under assumptions \ref{assum: c2} and \ref{assum: compact_2} we have
    \begin{gather*}
        p_{Y_t | X_t}(y_t | x_t)  = (p_{Y | X_t}(\cdot | x_t) \ast \varphi_\sigma)(y_t) \\
        \partial_{x_t} p_{Y_t | X_t}(y_t | x_t) = (\partial_{x_t} p_{Y| X_t}(\cdot | x_t) \ast \varphi_\sigma)(y_t)
    \end{gather*}
\end{lemma}
\begin{proof}
    For this proof, we drop our convention of denoting the probability distribution of a random variable via the name of its density’s argument.
    \begin{align*}
        p_{Y_t | X_t}(y_t | x_t)  &=
        \\ &= \int p_{Y, Y_t| X_t}(y, y_t | x_t) dy 
        \\ &= \int p_{Y | X_t}(y | x_t) p_{Y_t |Y, X_t}(y_t | y, x_t) dt
        \\ &= \int p_{Y | X_t}(y | x_t) p_{Y_t |Y}(y_t | y) dy
    \end{align*}
    Since $Y_t |Y$ has normal distribution with mean $y$ and variance $\sigma^y(t)^2$:
    \begin{align*}
        &= \int p_{Y | X_t}(y | x_t) \varphi_\sigma(y_t - y)  dy 
        \\ &= (p_{Y | X_t}(\cdot | x_t) \ast \varphi_\sigma)(y_t)
    \end{align*}
    where $\varphi_\sigma$ is a Gaussian kernel with variance $\sigma^y(t)^2$.
    Moreover, under the assumptions of the lemma we can exchange the differentiaion and integration. Therefore
    \begin{align*}
        \partial_{x_t} p_{Y_t | X_t}(y_t | x_t) &= \partial_{x_t} \int p_{Y | X_t}(y | x_t) \varphi_\sigma(y_t - y) dy   
        \\ &=  \int \partial_{x_t} p_{Y | X_t}(y | x_t) \varphi_\sigma(y_t - y) dy  
        \\ &= (\partial_{x_t} p_{Y | X_t}(\cdot | x_t) \ast \varphi_\sigma)(y_t)
    \end{align*}
\end{proof}

\begin{lemma}
    \label{lemma: sup_norm}
    Let $f$ be a $C^1$-function on a compact domain $\mathcal{X}$ and let $\varphi_\sigma$ be a Gaussian kernel with variance $\sigma^2$. Then there exists a function $E: \mathbb{R} \xrightarrow{} \mathbb{R}$, which is monotonically decreasing to zero, such that
    \begin{gather*}
        \norm{(f \ast \varphi_\sigma) - f}_{\infty} \leq E(1/\sigma).
    \end{gather*}
\end{lemma}

\begin{proof}
    \begin{align*}
        & \phantom{=}|(f \ast \varphi_\sigma)(y) -f(y)| 
        \\&= \bigg| \int f(z)  \varphi_\sigma(z-y) dz - \int f(y) \varphi_\sigma(z-y) dz \bigg|
        \\&\leq  \int |f(z) -f(y)| \varphi_\sigma(z-y) dz
    \end{align*}
    Since $f$ is  a $C^1$ function on a compact domain, it is Lipschitz and bounded (in absolute value) by some constant $M$. 
    Fix $\epsilon > 0$, and let $L$ denote the Lipschitz constant of $f$.
    We have that $|f(z) -f(y)| < \epsilon$ whenever $\norm{z-y} < \epsilon / L$.
    Let $B_y( \epsilon / L ) := \{z \in \mathcal{X} : \norm{z-y} < \epsilon / L \}$ be a ball of radius $\epsilon / L$ around $y$.
    Then
    \begin{align*}
        &\phantom{=}\int |f(z) -f(y)| \varphi_\sigma(z-y) dz 
        \\ &\! \begin{aligned}
            = \int_{B_y( \epsilon / L )}& |f(z) -f(y)| \varphi_\sigma(z-y) dz 
            \\ &+ \int_{\mathcal{X} \setminus B_y( \epsilon / L )} |f(z) -f(y)| \varphi_\sigma(z-y) dz
        \end{aligned}
        \\ &\leq \epsilon +  \int_{\mathcal{X} \setminus B_y( \epsilon / L )}  2M \varphi_\sigma(z-y) dz
       \\ &= \epsilon +  2M P \left( |Z_\sigma| > \frac{\epsilon}{L} \right)
    \end{align*}
    where $Z_\sigma$ is a normally-distributed random variable with mean zero and variance $\sigma^2$.
    By the Chernoff bound, we have
    \begin{align*}
        \leq  \epsilon +  4M \exp \left( -\frac{\epsilon^2}{2L^2 \sigma^2}\right).
    \end{align*}
    
    \noindent Define $E_\epsilon(1/\sigma) :=   \epsilon +  4M \exp \left( -\frac{\epsilon^2}{2L^2 \sigma^2}\right)$. Observe that $E_\epsilon: \mathbb{R}_+ \xrightarrow{} \mathbb{R}$ is monotonically decreasing to $\epsilon$.
    Moreover 
    $$ \norm{(f \ast \varphi_\sigma) - f}_{\infty} \leq E_{\epsilon}(1/\sigma). $$
    Now let $A := [0,1]$ and define 
    $$E(1/\sigma) := \min_{\epsilon \in A} E_\epsilon(1/\sigma).$$
    Notice that the above minimum is achieved, since $A$ is compact and for a fixed $\sigma$, the function $\epsilon \mapsto E_{\epsilon}(1/\sigma)$ is continuous.

    We will prove that $E$ is a monotonically decreasing to zero and upper-bounds $\norm{(f \ast \varphi_\sigma) - f}_{\infty}$.
    Firstly, it is clear that $E(x) \to 0$ as $x \to \infty$, since for all $\epsilon \in A$ we have $\lim_{x \to \infty} E(x) \leq \lim_{x \to \infty} E_{\epsilon}(x)=\epsilon$ . 
    Secondly, suppose $a < b$, and let $\epsilon_a$ be such that $E(a) = E_{\epsilon_a}(a)$. Then
    $$E(b) = \inf_{\epsilon \in A} E_{\epsilon}(b) \leq E_{\epsilon_a}(b) < E_{\epsilon_a}(a) = E(a).$$
    Therefore $E$ is monotonically decreasing.
    Finally since for all $\epsilon > 0$
    $$ \norm{(f \ast \varphi_\sigma) - f}_{\infty} \leq E_{\epsilon}(1/\sigma). $$
    Taking minimum over $\epsilon \in A$ on both sides we obtain 
    $$ \norm{(f \ast \varphi_\sigma) - f}_{\infty} \leq E(1/\sigma). $$
\end{proof}

\begin{lemma}
    \label{lemma: Lpz}
    Let $f$ be a $C^1$ function on a compact domain and let $Z$ be a random variable with mean $\mu$ and variance $\sigma^2$.
    Then
    \begin{gather*}
        \mathbb{E}_Z[(f(\mu) - f(Z))^2] \leq L^2 \sigma^2
    \end{gather*}
    where $L$ denotes the Lipschitz constant of $f$.
\end{lemma}
\begin{proof}
Since  $f$ is a $C^1$ function on a compact domain it is Lipschitz with some Lipschitz constant $L$. Therefore
\begin{align*}
    \mathbb{E}_Z[(f(\mu) - f(Z))^2] 
    \leq L^2\mathbb{E}_Z[(\mu - Z)^2]
    \leq L^2 \sigma^2
\end{align*}    
\end{proof}

\begin{customthm}{3}
    Fix $t$, $x_t$ and $y$. Then under Assumptions \ref{assum: c2}, \ref{assum: lower_bound} and \ref{assum: compact_2}, there exists a function $E: \mathbb{R} \xrightarrow{} \mathbb{R}$ which is monotonically decreasing to zero, such that
    \begin{gather*}
        \mathbb{E}_{y_t \sim p(y_t|y)}[
            \norm{ \nabla_{x_t} \ln p(x_t|y_t) - \nabla_{x_t} \ln p(x_t|y)}_2^2
            ] \\
            \leq E(1/\sigma^y(t)).
    \end{gather*}
\end{customthm}
\begin{proof}
    For this proof, we drop our convention of denoting the probability distribution of a random variable via the name of its density’s argument.
    \begin{align*}
        \norm{ \nabla_{x_t} \ln p_{X_t | Y_t}(x_t|y_t) - \nabla_{x_t} \ln p_{X_t | Y}(x_t|y)}_2^2
    \\= \sum_{i=1}^{n_x} ( \partial^i_{x_t} \ln p_{X_t | Y_t}(x_t|y_t) - \partial^i_{x_t} \ln p_{X_t | Y}(x_t|y) )^2
    \end{align*}
    Therefore it is sufficient to prove the theorem in each dimension separately. Hence, without loss of generality, we may assume that $x_t \in \mathbb{R}$ and show
    \begin{gather*}
        \mathbb{E}_{y_t \sim p(y_t|y)}[
            ( \partial_{x_t} \ln p_{X_t | Y_t}(x_t|y_t) - \partial_{x_t} \ln p_{X_t | Y}(x_t|y) )^2
        ]  
        \\ \leq E(1/\sigma^y(t)).
    \end{gather*}
    By Bayes's rule we have
    \begin{align*}
        &\partial_{x_t} \ln p_{X_t | Y_t}(x_t|y_t)  =  \partial_{x_t} \ln p_{Y_t | X_t}(y_t | x_t) + \partial_{x_t} \ln p_{X_t}(x_t)
        \\ &\partial_{x_t} \ln  p_{X_t | Y}(x_t|y)  = \partial_{x_t} \ln p_{Y | X_t}(y | x_t) + \partial_{x_t} \ln p_{X_t}(x_t).
    \end{align*}
    Therefore,
    \begin{align*}
        &(  \partial_{x_t} \ln p_{X_t | Y_t}(x_t|y_t)  - \partial_{x_t} \ln p_{X_t | Y}(x_t|y) )^2
        \\ &= (\partial_{x_t} \ln p_{Y_t | X_t}(y_t | x_t)- \partial_{x_t} \ln p_{Y | X_t}(y | x_t) )^2.
    \end{align*}
    To unclutter the notation, let $p(y | x) := p_{Y | X_t}(y | x)$ and $p_\sigma(y | x) :=  p_{Y_t | X_t}(y | x)$. Applying this notation:
    \begin{align*}
        \mathbb{E}_{y_t \sim p(y_t|y)}[
            (\partial_{x_t} \ln p_{Y_t | X_t}(y_t | x_t)- \partial_{x_t} \ln p_{Y | X_t}(y | x_t) )^2] \\
        = 
        \mathbb{E}_{y_t \sim p(y_t|y)}[
            (\partial_{x_t} \ln p_\sigma(y_t | x_t) - \partial_{x_t}  \ln p(y | x_t) )^2 ]
    \end{align*}
    Adding and subtracting $\partial_{x_t} \ln p(y_t | x_t)$ and using the triangle inequality:
    \begin{align*}
        \! \begin{aligned} 
            \leq &\mathbb{E}_{y_t \sim p(y_t|y)}[
                (\partial_{x_t} \ln p_\sigma(y_t | x_t) - \partial_{x_t} \ln p(y_t | x_t) )^2] \\
            &+ \mathbb{E}_{y_t \sim p(y_t|y)}[
                (\partial_{x_t} \ln p(y_t | x_t)- \partial_{x_t} \ln p(y | x_t) )^2]
        \end{aligned}
    \end{align*}
    We may bound the expectation by the supremum norm
    \begin{align*}
        \! \begin{aligned} 
            \leq & \norm{\partial_{x_t} \ln p_\sigma( \cdot | x_t) - \partial_{x_t} \ln p( \cdot | x_t) }_{\infty}^2 \\
            &+ \mathbb{E}_{y_t \sim p(y_t|y)}[(\partial_{x_t} \ln p(y_t | x_t)- \partial_{x_t} \ln p(y | x_t) )^2]
        \end{aligned}
    \end{align*}
    We will bound each of the summands separately. Firstly, by Assumption \ref{assum: c2}  $(y_t, x_t) \to p(y_t | x_t)$ is $C^2$ and therefore $(y_t, x_t) \to \partial_{x_t}p(y_t | x_t)$ is $C^1$. Moreover, since $\mathcal{X}$ is compact,  $y_t \to \partial_{x_t}p(y_t | x_t)$  is Lipschitz for some Lipschitz constant $L$. 
    Therefore, by Lemma \ref{lemma: Lpz},
    \begin{align*}
        \mathbb{E}_{y_t \sim p(y_t|y)}[ (\partial_{x_t} \ln p(y_t | x_t)- \partial_{x_t} \ln p(y | x_t) )^2] \leq  L^2 \sigma^y(t)^2.
    \end{align*} 
    To finish the proof, we need to bound $$ \norm{\partial_{x_t} \ln p_\sigma( \cdot | x_t) - \partial_{x_t} \ln p( \cdot | x_t) }_{\infty}^2 $$
    First, we apply the chain rule
    \begin{align*}
        &\phantom{=}\norm{\partial_{x_t} \ln p_\sigma( \cdot | x_t) - \partial_{x_t} \ln p( \cdot | x_t) }_{\infty}^2 
        \\ &=  \norm{ 
            \frac{\partial_{x_t} p_{\sigma}( \cdot | x_t)}{ p_{\sigma}( \cdot | x_t)}  
            - \frac{\partial_{x_t} p( \cdot | x_t)}{ p( \cdot | x_t)} 
        }_{\infty}^2 
    \end{align*}
    Adding and subtracting $\frac{\partial_{x_t} p_{\sigma}( \cdot | x_t)}{ p( \cdot | x_t)} $:
    \begin{align*}
        \ \! &\begin{aligned}
             \leq &\norm{                  
            \frac{\partial_{x_t} p_{\sigma}( \cdot | x_t)}{ p_{\sigma}( \cdot | x_t)}  
            - \frac{\partial_{x_t} p_{\sigma}( \cdot | x_t)}{ p( \cdot | x_t)} 
        }_{\infty}^2 
        \\ &+  \norm{ 
            \frac{\partial_{x_t} p_{\sigma}( \cdot | x_t)}{ p( \cdot | x_t)}  
            - \frac{\partial_{x_t} p( \cdot | x_t)}{ p( \cdot | x_t)} 
        }_{\infty}^2 
        \end{aligned}
        \\ \! &\begin{aligned}
             = &\norm{ 
            \frac{\partial_{x_t} p_{\sigma}( \cdot | x_t)[ p( \cdot | x_t) - p_{\sigma}( \cdot | x_t)]}
            { p_{\sigma}( \cdot | x_t) p( \cdot | x_t)}  
        }_{\infty}^2 
        \\ &+  \norm{ 
            \frac{\partial_{x_t} p_{\sigma}( \cdot | x_t) - \partial_{x_t} p( \cdot | x_t)}
            { p( \cdot | x_t)}  
        }_{\infty}^2 
        \end{aligned}
    \end{align*}
    By assumption \ref{assum: c2} and \ref{assum: compact_2} we have that $\partial_{x_t} p_{\sigma}( \cdot | x_t)$, $p_{\sigma}( \cdot | x_t)$ and  $p( \cdot | x_t)$ are continuous functions on a compact domain. Therefore, $\partial_{x_t} p_{\sigma}( \cdot | x_t)$ is bounded from above by some constant $M$. Moreover, by adding assumption \ref{assum: lower_bound} we obtain that $p_{\sigma}( \cdot | x_t)$ and  $p( \cdot | x_t)$ are bounded from below by some $\epsilon > 0$. Therefore 
    \begin{align*}
        \! &\begin{aligned}
            \leq &\norm{ 
           \frac{\partial_{x_t} p_{\sigma}( \cdot | x_t)[ p( \cdot | x_t) - p_{\sigma}( \cdot | x_t)]}
           { p_{\sigma}( \cdot | x_t) p( \cdot | x_t)}  
       }_{\infty}^2 
       \\ &+  \norm{ 
           \frac{\partial_{x_t} p_{\sigma}( \cdot | x_t) - \partial_{x_t} p( \cdot | x_t)}
           { p( \cdot | x_t)}  
       }_{\infty}^2 
       \end{aligned}
       \\ \! &\begin{aligned}
        \leq \frac{M}{\epsilon^2} &\norm{ p( \cdot | x_t) - p_{\sigma}( \cdot | x_t)
   }_{\infty}^2 
   \\ &+ \frac{1}{\epsilon}  \norm{ \partial_{x_t} p_{\sigma}( \cdot | x_t) - \partial_{x_t} p( \cdot | x_t)
   }_{\infty}^2 
   \end{aligned}
    \end{align*}
    Now by Lemma \ref{lemma: sup_norm} and \ref{lemma: blurring}
    \begin{align*}
        \leq \frac{M}{\epsilon^2} E_1(1/\sigma^y(t)^2) + \frac{1}{\epsilon}  E_2(1/\sigma^y(t)^2)
    \end{align*}
    where $E_1$ and $E_2$ are monotonically decreasing to zero.
    The theorem follows with $E(1/\sigma^y(t)^2) := \frac{M}{\epsilon^2} E_1(1/\sigma^y(t)^2) + \frac{1}{\epsilon}  E_2(1/\sigma^y(t)^2) + L^2 \sigma^y(t)^2$, which monotonically decreases to zero as $\sigma^y(t)^2$ decreases to zero.
\end{proof}

\section{Architectures and hyperparameters}
\label{appendix:hyperparams}
We used almost the same neural network architecture across all tasks and all estimators, so that we can compare the estimators fairly. The only difference between the score model for the diffusive estimators and the score model for the CDE estimator is that the former contains $6$ instead $3$ filters in the final convolution to account for the joint score estimation. This difference in the final convolution leads to negligible difference in the number of parameters, which is highly unlikely to have impacted the final performance. 

We used the basic version of the DDPM architecture with the following hyperparameters: channel dimension $96$, depth multipliers $[1, 1, 2, 2, 3, 3]$, $2$ ResNet Blocks per scale and attention in the final $3$ scales. The total parameter count is 43.5M. Song et al. \cite{song2021sde} report improved performance with the NCSN++ architecture over the baseline DDPM when training with the VE SDE. This claim is also supported by the work of Saharia et al. \cite{saharia2021sr3}. Therefore, adopting this architecture is likely to improve the performance of all estimators and lead to even more competitive performance over state-of-the-art methods. For all estimators, we concatenate the condition image $y$ or $y(t)$ with the diffused target $x(t)$ and pass the concatenated image as input to the score model for score calculation. In the super-resolution experiment, we first interpolate the condition to the same resolution as the target using nearest neighbours interpolation and then concatenate it with the target image. 

We used exponential moving average (EMA) with rate 0.999 and the same optimizer settings as in \cite{song2021sde}. Moreover, we used a batch size of $50$ for the super-resolution and edge to image translation experiments and a batch size of $100$ for the inpainting experiments.

\section{Extended visual results}

We provide additional samples in Figures \ref{fig:additional_sr}, \ref{fig:additional_inpainting} and \ref{fig:additional_shoe}.

\begin{figure*}
    \begin{center}
    \begingroup
    \setlength{\tabcolsep}{0pt}

    \begin{tabular}{ccccccc}
        Original image $x$ & Observation $ y$ & HCFlow & CDE & CDiffE & CMDE & VS-CMDE \\

        \includegraphics[width=.14\textwidth]{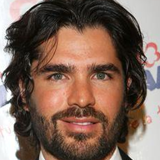} &   
        \includegraphics[width=.14\textwidth]{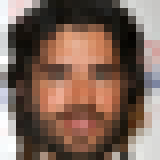} &
        \includegraphics[width=.14\textwidth]{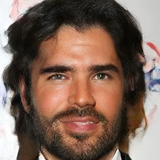} &
        \includegraphics[width=.14\textwidth]{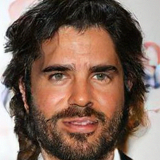} & 
        \includegraphics[width=.14\textwidth]{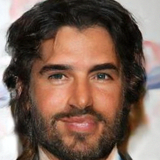} &
        \includegraphics[width=.14\textwidth]{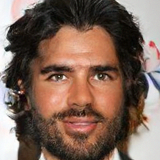} &
        \includegraphics[width=.14\textwidth]{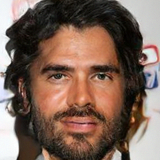}\\
        
        \includegraphics[width=.14\textwidth]{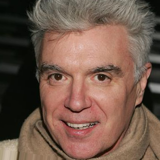} &   
        \includegraphics[width=.14\textwidth]{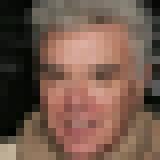} &
        \includegraphics[width=.14\textwidth]{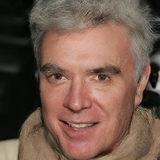} &
        \includegraphics[width=.14\textwidth]{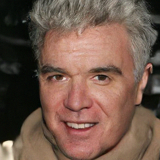} & 
        \includegraphics[width=.14\textwidth]{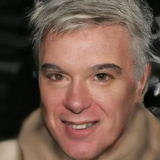} &
        \includegraphics[width=.14\textwidth]{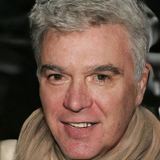} &
        \includegraphics[width=.14\textwidth]{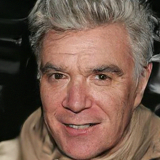}\\
        
        \includegraphics[width=.14\textwidth]{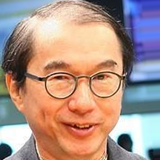} &   
        \includegraphics[width=.14\textwidth]{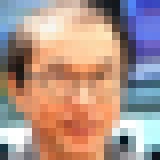} &
        \includegraphics[width=.14\textwidth]{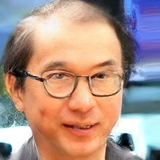} &
        \includegraphics[width=.14\textwidth]{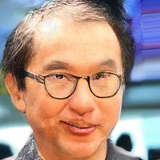} & 
        \includegraphics[width=.14\textwidth]{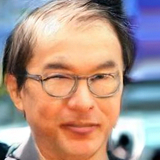} &
        \includegraphics[width=.14\textwidth]{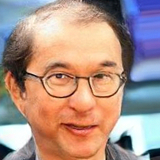} &
        \includegraphics[width=.14\textwidth]{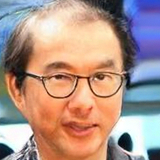}\\
        
        \includegraphics[width=.14\textwidth]{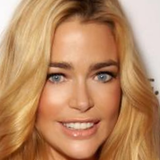} &   
        \includegraphics[width=.14\textwidth]{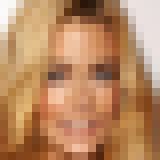} &
        \includegraphics[width=.14\textwidth]{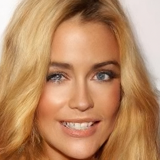} &
        \includegraphics[width=.14\textwidth]{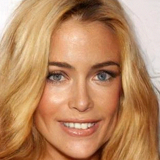} & 
        \includegraphics[width=.14\textwidth]{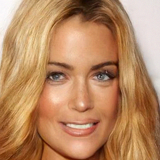} &
        \includegraphics[width=.14\textwidth]{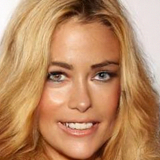} &
        \includegraphics[width=.14\textwidth]{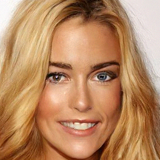}\\
        
        \includegraphics[width=.14\textwidth]{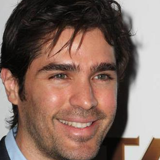} &   
        \includegraphics[width=.14\textwidth]{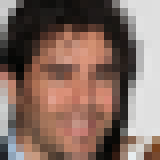} &
        \includegraphics[width=.14\textwidth]{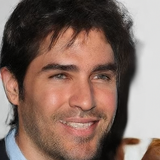} &
        \includegraphics[width=.14\textwidth]{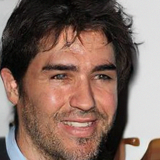} & 
        \includegraphics[width=.14\textwidth]{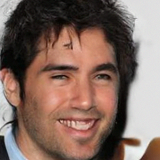} &
        \includegraphics[width=.14\textwidth]{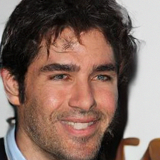} &
        \includegraphics[width=.14\textwidth]{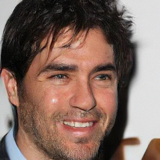}\\
        
        \includegraphics[width=.14\textwidth]{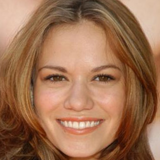} &   
        \includegraphics[width=.14\textwidth]{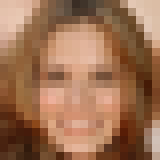} &
        \includegraphics[width=.14\textwidth]{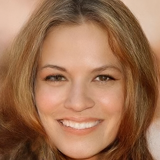} &
        \includegraphics[width=.14\textwidth]{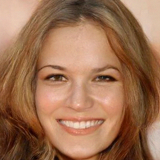} & 
        \includegraphics[width=.14\textwidth]{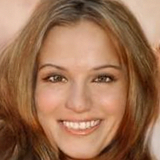} &
        \includegraphics[width=.14\textwidth]{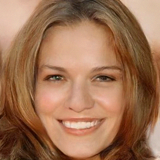} &
        \includegraphics[width=.14\textwidth]{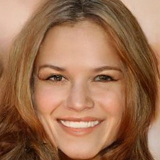}\\
        
        \includegraphics[width=.14\textwidth]{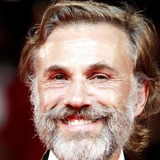} &   
        \includegraphics[width=.14\textwidth]{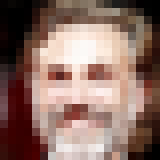} &
        \includegraphics[width=.14\textwidth]{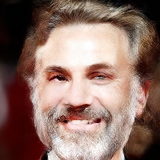} &
        \includegraphics[width=.14\textwidth]{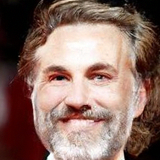} & 
        \includegraphics[width=.14\textwidth]{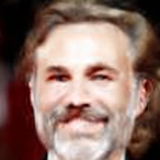} &
        \includegraphics[width=.14\textwidth]{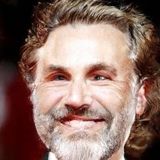} &
        \includegraphics[width=.14\textwidth]{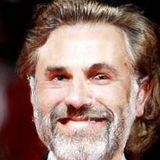}\\
        
        \includegraphics[width=.14\textwidth]{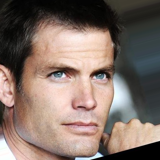} &   
        \includegraphics[width=.14\textwidth]{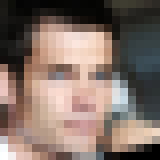} &
        \includegraphics[width=.14\textwidth]{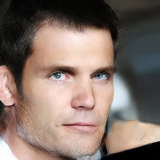} &
        \includegraphics[width=.14\textwidth]{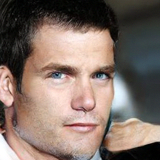} & 
        \includegraphics[width=.14\textwidth]{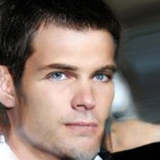} &
        \includegraphics[width=.14\textwidth]{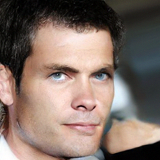} &
        \includegraphics[width=.14\textwidth]{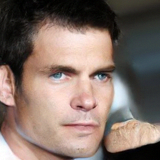}\\

    \end{tabular}
    \endgroup
    \end{center}
    \caption{Extended super-resolution results.}
    \label{fig:additional_sr}
\end{figure*}

\begin{figure*}
    \begin{center}
    \begingroup
    \setlength{\tabcolsep}{0pt}

    \begin{tabular}{ccccccc}
        Original image $x$ & Observation $ y$ & CDE & CDiffE & CMDE & VS-CMDE \\

        \includegraphics[width=.145\textwidth]{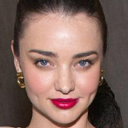} &   
        \includegraphics[width=.145\textwidth]{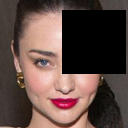} &
        \includegraphics[width=.145\textwidth]{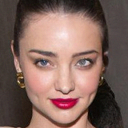} & 
        \includegraphics[width=.145\textwidth]{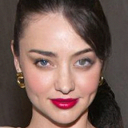} &
        \includegraphics[width=.145\textwidth]{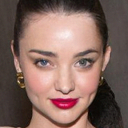} &
        \includegraphics[width=.145\textwidth]{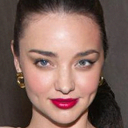}\\
        
        \includegraphics[width=.145\textwidth]{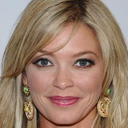} &   
        \includegraphics[width=.145\textwidth]{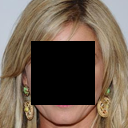} &
        \includegraphics[width=.145\textwidth]{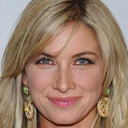} & 
        \includegraphics[width=.145\textwidth]{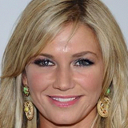} &
        \includegraphics[width=.145\textwidth]{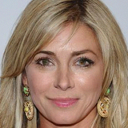} &
        \includegraphics[width=.145\textwidth]{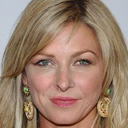}\\
        
        \includegraphics[width=.145\textwidth]{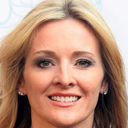} &   
        \includegraphics[width=.145\textwidth]{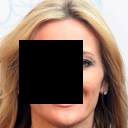} &
        \includegraphics[width=.145\textwidth]{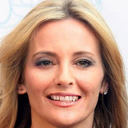} & 
        \includegraphics[width=.145\textwidth]{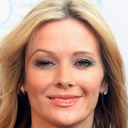} &
        \includegraphics[width=.145\textwidth]{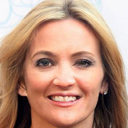} &
        \includegraphics[width=.145\textwidth]{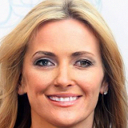}\\
        
        \includegraphics[width=.145\textwidth]{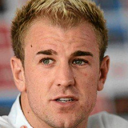} &   
        \includegraphics[width=.145\textwidth]{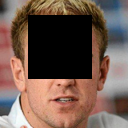} &
        \includegraphics[width=.145\textwidth]{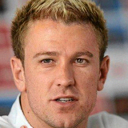} & 
        \includegraphics[width=.145\textwidth]{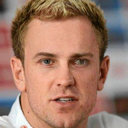} &
        \includegraphics[width=.145\textwidth]{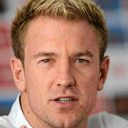} &
        \includegraphics[width=.145\textwidth]{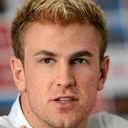}\\
        
        \includegraphics[width=.145\textwidth]{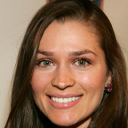} &   
        \includegraphics[width=.145\textwidth]{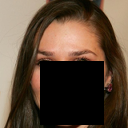} &
        \includegraphics[width=.145\textwidth]{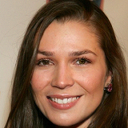} & 
        \includegraphics[width=.145\textwidth]{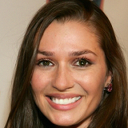} &
        \includegraphics[width=.145\textwidth]{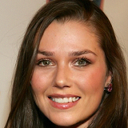} &
        \includegraphics[width=.145\textwidth]{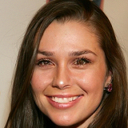}\\
        
        \includegraphics[width=.145\textwidth]{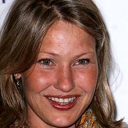} &   
        \includegraphics[width=.145\textwidth]{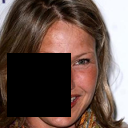} &
        \includegraphics[width=.145\textwidth]{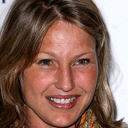} & 
        \includegraphics[width=.145\textwidth]{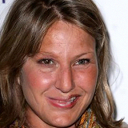} &
        \includegraphics[width=.145\textwidth]{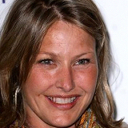} &
        \includegraphics[width=.145\textwidth]{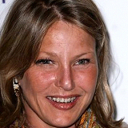}\\
        
        \includegraphics[width=.145\textwidth]{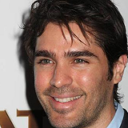} &   
        \includegraphics[width=.145\textwidth]{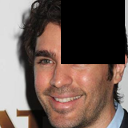} &
        \includegraphics[width=.145\textwidth]{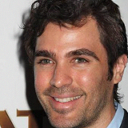} & 
        \includegraphics[width=.145\textwidth]{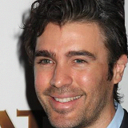} &
        \includegraphics[width=.145\textwidth]{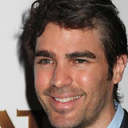} &
        \includegraphics[width=.145\textwidth]{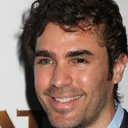}\\
        
        \includegraphics[width=.145\textwidth]{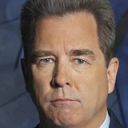} &   
        \includegraphics[width=.145\textwidth]{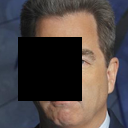} &
        \includegraphics[width=.145\textwidth]{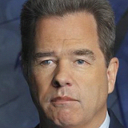} & 
        \includegraphics[width=.145\textwidth]{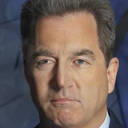} &
        \includegraphics[width=.145\textwidth]{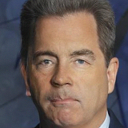} &
        \includegraphics[width=.145\textwidth]{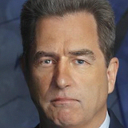}\\

    \end{tabular}
    \endgroup
    \end{center}
    \caption{Extended inpainting results.}
    \label{fig:additional_inpainting}
\end{figure*}

\begin{figure*}
    \begin{center}
    \begingroup
    \setlength{\tabcolsep}{2.5pt}
    \begin{tabular}{lcccccccc}
        \begin{tabular}{@{}l@{}}
            Original image $x$
            \\[25pt]
        \end{tabular}
         & \includegraphics[width=.08\textwidth]{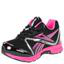} & \includegraphics[width=.08\textwidth]{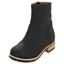} & \includegraphics[width=.08\textwidth]{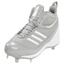} &
         \includegraphics[width=.08\textwidth]{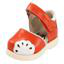} &
         \includegraphics[width=.08\textwidth]{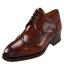} &
         \includegraphics[width=.08\textwidth]{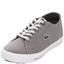} &
         \includegraphics[width=.08\textwidth]{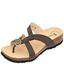} &
         \includegraphics[width=.08\textwidth]{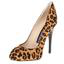} \\
        
        \begin{tabular}{@{}l@{}}
            Observation \\ $ y := Ax$
            \\[25pt]
        \end{tabular}
         & \includegraphics[width=.08\textwidth]{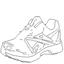} & \includegraphics[width=.08\textwidth]{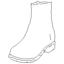} & \includegraphics[width=.08\textwidth]{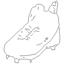} &
         \includegraphics[width=.08\textwidth]{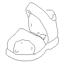} &
         \includegraphics[width=.08\textwidth]{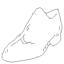} &
         \includegraphics[width=.08\textwidth]{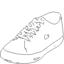} &
         \includegraphics[width=.08\textwidth]{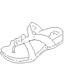} &
         \includegraphics[width=.08\textwidth]{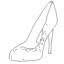} \\
        
         \begin{tabular}{@{}c@{}}
            CDE
            \\[25pt]
        \end{tabular} & 
            \includegraphics[width=.08\textwidth]{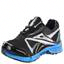} & \includegraphics[width=.08\textwidth]{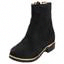} & \includegraphics[width=.08\textwidth]{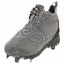} &
            \includegraphics[width=.08\textwidth]{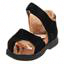} &
            \includegraphics[width=.08\textwidth]{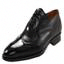} &
            \includegraphics[width=.08\textwidth]{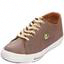} &
            \includegraphics[width=.08\textwidth]{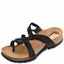} &
            \includegraphics[width=.08\textwidth]{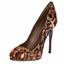} \\

        \begin{tabular}{@{}c@{}}
           CDiffE
            \\[25pt]
        \end{tabular} & 
            \includegraphics[width=.08\textwidth]{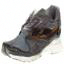} & \includegraphics[width=.08\textwidth]{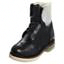} & \includegraphics[width=.08\textwidth]{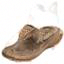} &
            \includegraphics[width=.08\textwidth]{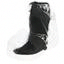} &
            \includegraphics[width=.08\textwidth]{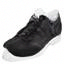} &
            \includegraphics[width=.08\textwidth]{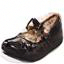} &
            \includegraphics[width=.08\textwidth]{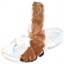} &
            \includegraphics[width=.08\textwidth]{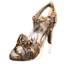} \\

        \begin{tabular}{@{}c@{}}
            CMDE
            \\[25pt]
        \end{tabular}  &  
            \includegraphics[width=.08\textwidth]{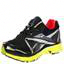} & \includegraphics[width=.08\textwidth]{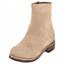} & \includegraphics[width=.08\textwidth]{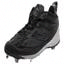} &
            \includegraphics[width=.08\textwidth]{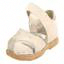} &
            \includegraphics[width=.08\textwidth]{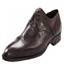} &
            \includegraphics[width=.08\textwidth]{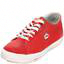} &
            \includegraphics[width=.08\textwidth]{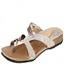} &
            \includegraphics[width=.08\textwidth]{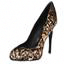} \\
            
        \begin{tabular}{@{}c@{}}
            VS-CMDE
            \\[25pt]
        \end{tabular}  &  
            \includegraphics[width=.08\textwidth]{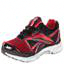} & \includegraphics[width=.08\textwidth]{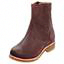} & \includegraphics[width=.08\textwidth]{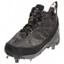} &
            \includegraphics[width=.08\textwidth]{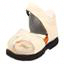} &
            \includegraphics[width=.08\textwidth]{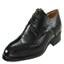} &
            \includegraphics[width=.08\textwidth]{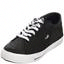} &
            \includegraphics[width=.08\textwidth]{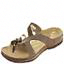} &
            \includegraphics[width=.08\textwidth]{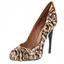} \\
    \end{tabular}
    \endgroup
    \end{center}
    \caption{Extended edge to shoe synthesis results.}
    \label{fig:additional_shoe}
\end{figure*}

\section{Potential negative impact}
The potential of negative impact of this work is the same as that of any work that advances generative modeling. Generative modeling can be used for the creation of deep-fakes which can be used for malicious purposes such as disinformation and blackmailing. However, research on generative modeling can indirectly or directly contribute to the robustification of deep-fake detection algorithms. Moreover, generative models have proven very useful in academic research and in industry. The potential benefits of generative modeling outweigh the potential threats. Therefore, the research community should continue to conduct research on generative modeling.

\end{document}